\documentclass{article}

\usepackage{microtype}
\usepackage{graphicx}
\usepackage{subfigure}
\usepackage{multirow}
\usepackage{amsmath}
\usepackage{amssymb}
\usepackage{algorithm}
\usepackage{algorithmic}
\usepackage{bm}
\usepackage{color}
\newtheorem{theorem}{Theorem}
\newtheorem{lemma}{Lemma}
\newtheorem{corollary}{Corollary}
\newtheorem{remark}{\noindent Remark}
\newenvironment{proof}{{\noindent\it Proof}\quad}{\hfill $\square$\par}
\newtheorem{proposition}{Proposition}
\usepackage{mathrsfs}
\usepackage{booktabs} 
\usepackage[T1]{fontenc}
\usepackage{hyperref}

\usepackage[accepted]{icml2020}

\icmltitlerunning{Meta Transition Adaptation for Robust Deep Learning with Noisy Labels}

\begin{document}

\twocolumn[
\icmltitle{Meta Transition Adaptation for Robust Deep Learning with Noisy Labels}



\icmlsetsymbol{equal}{*}

\begin{icmlauthorlist}

\icmlauthor{Jun Shu}{to}
\icmlauthor{Qian Zhao}{to}
\icmlauthor{Zongben Xu}{to}
\icmlauthor{Deyu Meng}{to,goo}
\end{icmlauthorlist}

\icmlaffiliation{to}{School of Mathematics and Statistics and Ministry of Education Key Lab of Intelligent Networks	and Network Security, Xian Jiaotong University, Xi'an, China}
\icmlaffiliation{goo}{Macau Institute of Systems Engineering, Macau University of Science and Technology,Macau,China}

\icmlcorrespondingauthor{Deyu Meng}{dymeng@mail.xjtu.edu.cn}

\icmlkeywords{Machine Learning, ICML}

\vskip 0.3in
\vspace{-2mm}
]



\printAffiliationsAndNotice{}  

\begin{abstract}
To discover intrinsic inter-class transition probabilities underlying data, learning with noise transition has become an important approach for robust deep learning on corrupted labels. Prior methods attempt to achieve such transition knowledge by pre-assuming strongly confident anchor points with 1-probability belonging to a specific class, generally infeasible in practice, or directly jointly estimating the transition matrix and learning the classifier from the noisy samples, always leading to inaccurate estimation misguided by wrong annotation information especially in large noise cases. To alleviate these issues, this study proposes a new meta-transition-learning strategy for the task. Specifically, through the sound guidance of a small set of meta data with clean labels, the noise transition matrix and the classifier parameters can be mutually ameliorated to avoid being trapped by noisy training samples, and without need of any anchor point assumptions. Besides, we prove our method is with statistical consistency guarantee on correctly estimating the desired transition matrix. Extensive synthetic and real experiments validate that our method can more accurately extract the transition matrix, naturally following its more robust performance than prior arts. Its essential relationship with label distribution learning is also discussed, which explains its fine performance even under no-noise scenarios.
\end{abstract}\vspace{-2mm}

\section{Introduction}\label{introduction}
While deep neural networks (DNNs) have recently obtained remarkable success on various applications \cite{krizhevsky2012imagenet,he2016deep}, its performance largely relies on a pre-collected large-scale dataset with high quality of human annotations. In real-world applications, however, it is notoriously expensive both in time and money to achieve such data. In real practice, instead, data labels are always collected by coarse annotation sources, like crowdsourcing systems \cite{bi2014learning} or search engines \cite{xiao2015learning,liang2016learning}, naturally resulting in the noisy (incorrect) label problem in training data. Learning with such biased training data easily encounters the overfitting issue, thus hampering the generalization performance of the utilized learning regimes \cite{zhang2016understanding}.

\begin{figure}[t]
	\centering
	\vspace{-1mm}
	\includegraphics[width=1.05\linewidth]{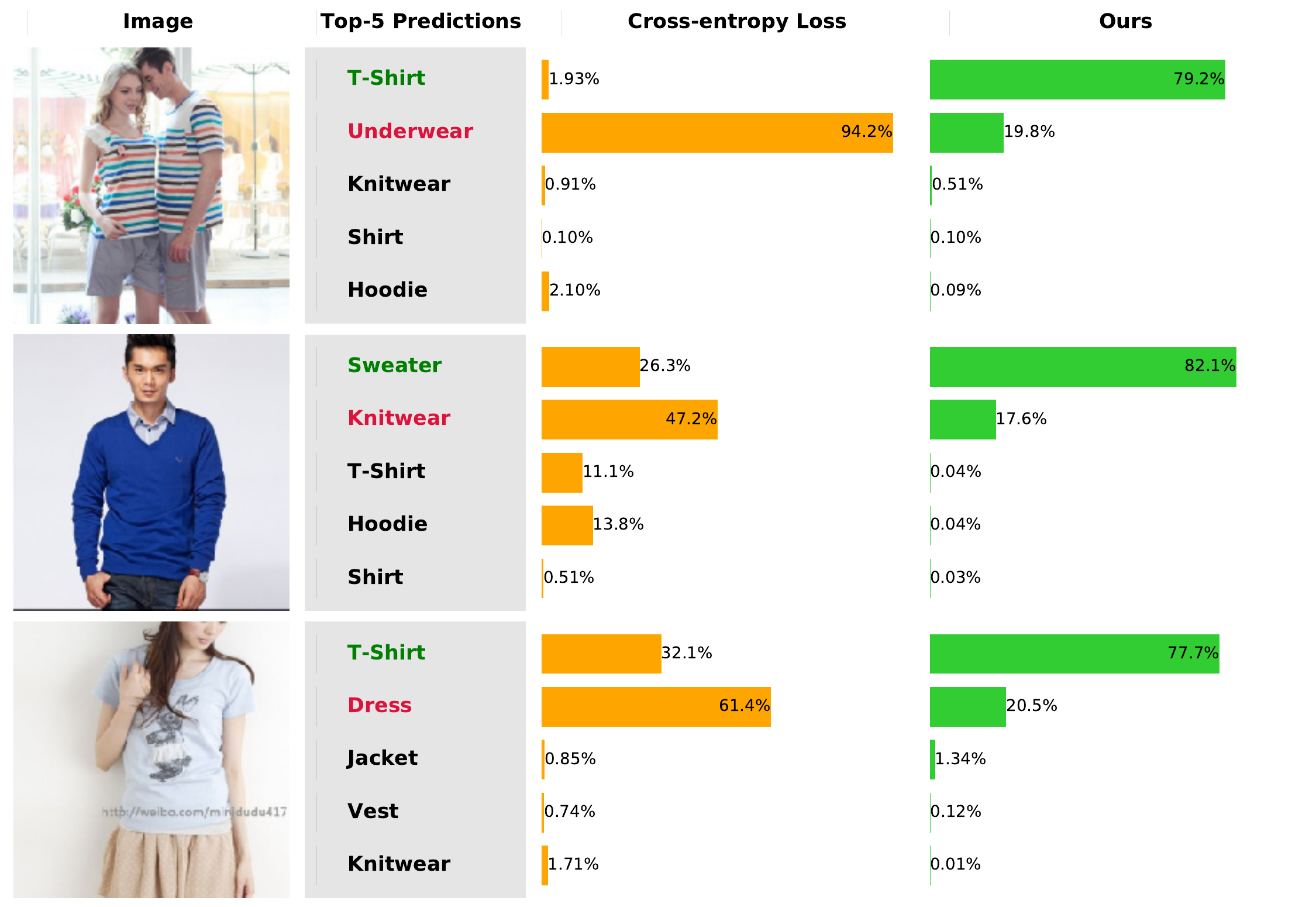}
	\vspace{-6mm}
	\caption{Examples of generated labels by training with cross-entropy loss (third column) and the proposed method (fourth column) on Clothing1M. The second column show the top-5 predictions of two methods, and also can be seen as ambiguous classes of the image most possibly belongs to. The green and red labels denote the clean and the noisy labels, respectively.}\label{fig1}
	\vspace{-6mm}
\end{figure}

The commonly used approach against this robust learning issue is to select confident examples and remove suspect ones \cite{chang2017active} or to correct noisy labels to their more possibly true labels \cite{arazo2019unsupervised}. These methods, however, implicitly assume a sample belongs only to one class, but neglect the intrinsic labeling noise insight in real-world that there are essential ambiguities among various sample categories. While such ``noisy label'' are useful to deliver intrinsic knowledge of inter-class transition principle naturally existed in data annotation, just coarsely removing noisy samples or transferring a noisy label to another ignores this label noise generation clue, and thus makes them still have room for further performance improvement.

Such label ambiguity issue can be easily understood by seeing Fig.\ref{fig1}, where the samples are from Clothing1M \cite{xiao2015learning}, a large-scale clothing dataset by crawling images from several online shopping websites. It represents a typical real-world label corruption scenario: there exists an unknown noise transition matrix to flip the more possibly true label to other less possible ones with probability, and thus to produce noisy labels. Directly training a DNN classifier by taking given sample labels as deterministic, the top-1 predictions tend to be consistent with the noisy labels, naturally conducting overfitting issue, as clearly shown in the third column of Fig.\ref{fig1}. Achieving the underlying noise transition matrix is thus expected to be helpful for alleviating such robust issue by thoroughly extracting the real noisy label distribution and ameliorating the quality of trained classifier (as depicted in the fourth column of the figure).

Pervious methods for noise transition matrix estimation can be roughly summarized as two solutions. One is to estimate this matrix on pre-assumed anchor points, i.e., sample(s) certainly belonging to each class, in advance, and subsequently fix it to train the classifier. However, such prior knowledge \cite{scott2013classification,patrini2017making} are generally infeasible in practice. The other solution is to jointly estimate the noise transition matrix and the classifier parameters in a unified framework \cite{sukhbaatar2014training,goldberger2016training}. Although it avoids the anchor point assumption, it always obtains inaccurate estimation misguided by wrong annotation information especially in large noise cases, as clearly depicted in our experiments.

Against the above issues, this paper proposes a new meta-transition-learning strategy against the noisy labels. The main idea is to leverage a small set of meta-data with clean labels to guide the estimation of noise transition. In summary, this study mainly made three-fold contributions. \vspace{-1mm}
 \begin{itemize}
\item  We propose a new learning strategy to estimate the noise transition matrix in a meta-learning manner. Under the guidance of a small set of meta data with clean labels,
    the noise transition matrix and the classifier parameters can be mutually ameliorated to avoid being trapped by noisy training samples, and without need of any anchor point assumptions.\vspace{-1mm}
\item  We show that our method can finely estimate the desired transition matrix under the guidance of the meta data with a statistical consistency guarantee. Comprehensive synthetic and real experiments validate that our method can more accurately extract the transition matrix underlying data, naturally following its more robust performance, than previous SOTA methods.\vspace{-1mm}
\item We discuss the essential relationship between our method and label distribution learning, which explains its fine performance even under no-noise scenarios. Experiments on out-of-training-distribution behavior and adversarial attacks shows that our method can bring model better generalization and robustness.
 \end{itemize}
The paper is organized as follows. Section \ref{related_work} reviews the related works. Section \ref{section3} introduces the proposed meta learning method, as well as some of its fine statistical properties. Section \ref{experiment} demonstrates experimental results. Section \ref{ldl} discusses the relationship between our method and label distribution learning, and a conclusion is finally made.

\vspace{-2mm}
\section{Related Work}\label{related_work}
\textbf{Learning with Noise Transition.} Transition matrix reflects the probabilities that most probable true labels flip into other ``noise'' ones, which has been previously employed to modify loss functions to help improve the training performance \cite{natarajan2013learning,scott2015rate}. There exist mainly two approaches to estimate the noise transition matrix.
One is to leverage a two-step solution to pre-estimate noise transition with the anchor point prior assumption and then use it to train the classifier. E.g., \cite{patrini2017making} proposed a theoretically sound loss correction method for the task by using pre-calculated noise transition knowledge, which are obtained on heuristically collected anchor points from the unsupervised dataset. Afterwards, GLC \cite{hendrycks2018using} used a small set of pre-assumed clean-label samples to estimate the noise transition to further improve estimation stability. These methods, however, require to pre-specify instances belonging to a special class with probability exactly or at least very approaching one, which is always an infeasible task in practice. The approximate used anchor points always lead to inaccurate estimation of the matrix, and thus hamper the subsequent training accuracy.

The other approach is to jointly estimate the noise transition matrix and the classifier parameter in a unified framework without employing anchor points. \citet{sukhbaatar2014training} first learned a linear layer with a trace constrained, which pushes the linear layer to be interpreted as the transition matrix between the true and noisy labels. \cite{jindal2016learning} further ameliorated the result by additional dropout regularization. Subsequently, S-Model \cite{goldberger2016training} modelled the noise transition with a Softmax layer beyond linear.
Recently, T-Revision \cite{xia2019anchor} introduced a slack variable to revise the pre-estimated matrix and validate the revision on noisy validation set. Albeit with concise calculation paradigm, the accuracy of these methods tend to be hampered misguided by noisy labels, especially in heavy noise rate cases, as clearly shown in our experiments.

\textbf{Other methods of learning with noisy labels.} We also shortly introduce two typical strategies for handling noisy labels issue: label correction and reliable example selection approaches. The former aims to correct noisy labels to their true ones via an inference step, like directed graphical models \cite{xiao2015learning}, conditional random fields \cite{vahdat2017toward} or knowledge graphs \cite{li2017learning}. \cite{tanaka2018joint} used the network outputs to predict hard or soft labels. Decouple \cite{malach2017decoupling} selected the samples with different label predictions of two networks, while  Co-teaching \cite{han2018co} selected its small-loss samples as clean samples for each network. INCV \cite{chen2019understanding} randomly divided the noisy data and then utilized cross-validation to identify clean samples by removing large-loss samples at each iteration. The other reliable example selection approach mainly adopts sample re-weighting schemes by imposing weights on samples based on their reliability for training. Typical methods include SPL \cite{kumar2010self} and its extensions \cite{jiang2014easy,jiang2014self,meng2017theoretical}, by reducing effects of examples with large losses, and pay more attention to easy samples with smaller losses. Some other methods along this line include iterative reweighting strategy \cite{zhang2018generalized},  Bayesian latent variables inference \cite{wang2017robust} and so on.

Recently, some works try to combine advantages of above two approaches. For example, SELFIE \cite{song2019selfie} trained the network on selectively refurbished false-labeled samples that can be corrected with a high precision together with small-loss ones. \cite{arazo2019unsupervised} used a two-component mixture model to character the loss distribution of clean and noisy samples in an unsupervised way, and used mixup data augmentation to achieve noisy label correction.
\cite{shen2019learning} proposed to iteratively minimize the trimmed loss to select samples with lowest current loss and retrain a model on only these samples, which is proved that recovers the ground truth in generalized linear models.

\textbf{Meta learning methods.} Inspired by meta-learning developments \cite{schmidhuber1992learning,thrun2012learning,finn2017model,shu2018small,shu2019meta}, recently some methods were proposed to make DNNs robust to label noise. However, existing methods focus on learning an adaptive weighting scheme imposed on data to make the learning more automatic and reliable. Typical methods along this line include MentorNet \cite{jiang2018mentornet}, L2RW \cite{ren2018learning} and Meta-Weight-Net \cite{shu2019meta}. This paper can be seen as the first exploration of meta learning on fitting noise transition information.

\vspace{-2mm}
\section{Meta Transition Adaptation Method}\label{section3}
\subsection{Preliminaries}
We consider the problem of $c$-class classification. Let  $\mathcal{X} \subset \mathbb{R}^d$ be the feature space,  $\mathcal{Y}=\{1,2,\cdots,c\}$ be the label space, and $(X,Y), (X,\widetilde{Y}) \in (\mathcal{X},\mathcal{Y})$ denote the underlying data distributions with true and noisy labels.
In practice, we assume that the labels of the collected training examples are independently corrupted from the true label distribution. Thus what we can obtain are the noisy training samples $\widetilde{D}= \{(x_i,\widetilde{y}_i)\}_{i=1}^N$, corresponding to the latent true data samples $\mathcal{D}= \{(x_i,y_i)\}_{i=1}^N$. The two datasets are i.i.d. drawn from true and noisy data distributions $p_{X\widetilde{Y}}$ and $p_{XY}$, respectively.

Assume our classifier model is a DNN architecture with $d$ layers comprising a transformation $\bm{h}: \mathcal{X}\rightarrow \mathbb{R}^c$, where $\bm{h}=\bm{h}^{(d)}\circ\bm{h}^{(d-1)}\circ\cdots\circ\bm{h}^{(1)}$ is the composition of a series of intermediate transformations layers $\bm{h}^{(i)}$. Each $\bm{h}^{(i)}$ is defined  as:
\begin{align*}
\bm{h}^{(i)}(z) &= \sigma(W^{(i)}z), i=1,\cdots,d-1,\\
\bm{h}^{(d)}(z) & = W^{(d)}z,
\end{align*}
where $W^{(i)}$ denote the classifier parameters to be estimated \footnote{Here, we omit the bias vector in each layer.}, and $\sigma$ is the activation function such as ReLU \cite{glorot2011deep}. We assume that the output layer is a Softmax layer, and then the output is $f_i(x)=\frac{\exp(\bm{h}_i(x))}{\sum_{k=1}^c\exp(\bm{h}_k(x))},i=1,2,\cdots,c$, and the predicted label is thus given by $\arg\max_{i=1,2,\cdots,c}f_i(x)$.
The Softmax output can be interpreted as a $c$-dimensional vector approximating the class-conditional probabilities $p(Y|X)$. We denote it by $\hat{p}(Y|X)$, also written as $\hat{p}(Y|X)=f(X)$. The expected risk on clean data is defined as \cite{bartlett2006convexity}:
\begin{align} \label{eq1}
R(f) = \mathbb{E}_{(X,Y)\sim P_{XY}} \ell (f(X),Y),
\end{align}
where $\ell:  \mathbb{R}^c \times\mathcal{Y}\rightarrow \mathbb{R}$ is the loss function.

Since the distribution $P_{XY}$ is usually unknown, we use the empirical risk $R_{N}(f)$  over dataset $D$ to approximate $R(f)$,
\begin{align}\label{eq3}
 R_N(f) = \frac{1}{N}\sum_{i=1}^N \ell (f(x_i^{(m)}),y_i^{(m)}).
\end{align}
In this study, we assume there are label transition probabilities between different classes, as commonly adopted in the previous works \cite{natarajan2013learning,sukhbaatar2014training, patrini2017making,goldberger2016training}. The probability of each label $y$ in the training set flipping to $\widetilde{y}$ is expressed as $p(\widetilde{Y}=\widetilde{y}|Y=y)$.
We utilize a noise transition matrix $T\in [0,1]^{c\times c}$ \cite{van2017theory} to represent the probability $p(\widetilde{Y}=\widetilde{y}|Y=y)$, so that $T_{ij}=p(\widetilde{Y}=j|Y=i), \forall i,j$. The matrix is row-stochastic and not necessarily symmetric across the classes.

If we directly learn the classifier on the noisy data, we would obtain a class posterior predictor for noisy labels $p(\widetilde{Y}|X)$.
Noise transition matrix bridges $p(\widetilde{Y}|X)$ and the class posterior predictor for clean labels as follows:
\begin{align}\label{eqlabel}
p(\widetilde{Y}=j|X=x) = \sum_{i=1}^c T_{ij} p(Y=i|X=x),
\end{align}
and the corresponding matrix form can be written as $p(\widetilde{Y}|X)=T^Tp(Y|X)$. It is easy to observe that once the noise transition matrix is obtained, we can recover the desired estimator of class posterior predictor $p(Y|X)$ by the softmax output $f(x)$ through training the classifier $p(\widetilde{Y}|X)$, which is obtained by modifying the $p(Y|X)$ with $T$. Thus the expected risks with respect to noisy data is
\begin{align}\label{eq32}
&\overline{R}(f) = \mathbb{E}_{(X,\widetilde{Y})\sim P_{X\widetilde{Y}}} \ell (T^Tf(X),\widetilde{Y}),
\end{align}
and the empirical risk over noisy dataset $\widetilde{D}$ is
\begin{align}\label{eq31}
\overline{R}_{N}(f)=-\frac{1}{N}\sum_{i=1}^N \ell (T^Tf(x_i),\widetilde{y}_i).
\end{align}
%
It has been exploited to build a classifier-consistent algorithm \cite{patrini2017making,xia2019anchor}, i.e., once the noise transition is obtained, by increasing the size of noisy examples, the learned classifier of Eq.(\ref{eq31}) will converge to the optimal classifier learned by clean examples of Eq.(\ref{eq3}).

\subsection{Existing Estimation Methods}
The success of classifier-consistent algorithms depends on the accurate estimation of the transition matrix. There exist two strategies to learn the matrix. One is a two-stage regime to utilize anchor point assumption \cite{patrini2017making} to pre-estimate the noise transition and then use it to train the classifier. By assuming instance $x$ is the anchor point for class $i$ if $p(Y=i|X=x)=1$, and it holds that
\begin{align}
p(\widetilde{Y}=i|X=x) = \sum_{k=1}^c T_{kj}p(Y=k|X=x)= T_{ij},
\end{align}
since $p(Y=k|X=x)=0, \ \forall k \neq i$. Thus if $p(\widetilde{Y}|X)$ can be approximated by the softmax output $f(x)$ (i.e., $\hat{p}(\widetilde{Y}|X)$), $T$ can be obtained via estimating the noisy class posterior probabilities for anchor points. To pre-attain such anchor points, \citet{patrini2017making} designed certain heuristic strategy on unsupervised samples, and \citet{hendrycks2018using} used a small set of clean samples to simulate anchor points. Once obtaining $T$, it can recover $p(Y|X)$ by optimizing Eq.(\ref{eq31}) according to classifier-consistent algorithms.
However, the prior on anchor points is always hard to achieve in practice, increasing the difficulty of using them.

The other is a one-stage strategy to jointly estimate the noise transition matrix and the classifier parameters in a unified framework, and the noise transition $T$ can be modeled as a constrained linear layer \cite{sukhbaatar2014training} or a Softmax layer \cite{goldberger2016training}.
For example, S-Model \cite{goldberger2016training} modeled the matrix by adding another Softmax layer to the network, whose parameters can be learned using standard techniques for neural network training. Thus, they trained the classifier and Softmax layer simultaneously directly on the noisy data. At test time, they removed the adding softmax layer and used the classifier to predict the true labels.
Recently, \citet{xia2019anchor} proposed a T-Revision method to approximate $T$ by gradually ameliorating a slack variable imposed on it, together with updating the classifier parameters. The limitation of these methods mainly lies on its easy misguidance by the noisy annotations, especially in large noise cases, since they are directly trained on them.

\begin{algorithm}[t]
	\vspace{0mm}
	\renewcommand{\algorithmicrequire}{\textbf{Input:}}
	\renewcommand{\algorithmicensure}{\textbf{Output:}}
	\caption{The proposed learning Algorithm}
	\label{alg1}
	\begin{algorithmic}[1]  \small
		\REQUIRE  Training data $\widetilde{D}$, meta data $D_{meta}$, batch size $n,m$, max iterations $Iter$.
		\ENSURE  Classifier $f$ parameter $\mathbf{W}$, noise transition matrix $T$.
		\STATE Initialize classifier parameter $\mathbf{W}^{(0)}$ and noise transition matrix parameter  $T^{(0)}$.
		\FOR{$t=0$ {\bfseries to} $Iter-1$}
		\STATE $\{x,y\} \leftarrow$ SampleMiniBatch($\widetilde{D},n$).
		\STATE $\{x^{(m)},y^{(m)}\} \leftarrow$ SampleMiniBatch($D_{meta},m$).
		\STATE Update $T^{(t)}$ by Eq. (\ref{theta}).
		\STATE Update $\mathbf{W}^{(t)}$ by Eq. (\ref{wt}).
		\ENDFOR
	\end{algorithmic}
\end{algorithm}

\subsection{Meta Transition Adaptation Method}
To alleviate the aforementioned issues of the current methods, we propose a new learning strategy, which utilizes a small set of meta data with clean labels to guide the estimation of the noise transition matrix.
Specifically, we leverage a small set of meta data set  $D_{meta}=\{(x_i^{(m)},y_i^{(m)})\}_{i=1}^M$ with clean labels, representing the meta-knowledge of underlying label distribution of clean samples, where $M$ is the number of meta-samples, and $M \ll N$. Note that the data can always be attainable in practice as compared with infeasible anchor point priors and large collection of clean samples required in traditional DL methods. Then we formulate the following bi-level minimization problem to jointly estimate the noise transition matrix and learn the classifier parameters:
\begin{align}\label{eqfw}
T^* = \mathop{\arg\min}_{T\in [0,1]^{c\times c}} \frac{1}{M} \sum_{i=1}^M \mathcal{L}_{M} (f^*_T(x_i^{(m)}),y_i^{(m)}),\\ \label{eqthe}
f^*_T=
\mathop{\arg\min}_{f\in \mathcal{F}} \frac{1}{N}\sum_{i=1}^N \ell (T^Tf(x_i),\widetilde{y}_i),
\end{align}
where $\mathcal{F}$ and $\mathcal{L}_{M}$ denote the hypothesis space of $f$ and the loss function imposed on meta data, respectively. $f^*_T$ represents the optimal classifier that minimizes Eq.(\ref{eqthe}) on the noisy dataset $\widetilde{D}$ while depends on $T$ ($f_T$ is the functional operator with parameter $T$). We use cross-entropy (CE) loss as training and meta loss in all our experiments. Note that we treat $T$ as training hyper-parameter, and the estimation of it should minimize the loss on meta data in a meta-learning manner \cite{finn2017model,shu2019meta}.

We have further proved that our method can recover the ground-truth noise transition matrix with meta loss in probability under some mild conditions, and our method is thus with statistical consistency property. All theoretical results and proof details are listed in supplementary material.

\vspace{-1mm}\subsection{Generalization Error}\vspace{-1mm}
We then show an upper bound for the estimation error supposed that we obtain the ground-truth noise transition matrix by using Rademacher complexity \cite{mohri2018foundations}.
\begin{theorem}\label{TH3}
Let $\mathcal{H}$ be the class of real-valued networks of depth $d$ over the domain $\mathcal{X}=\{x: \|x\|\leq B\}$, where each parameter matrix $W^{(i)}$ is with Frobenius norm at most $M_i$, and the activation function $\sigma$ is 1-Lipschitz, positive-homogeneous and applied element-wise (such as the ReLU).  Suppose the loss function be the CE loss, and then for any $\delta \in (0,1)$, with the probability at least $1-\delta$, it holds that:\vspace{-7mm}

\small
	\begin{align*}
	\overline{R}(f)\leq& \overline{R}_N(f)+  \frac{2cMB\left(\sqrt{2\log(2)d}+1\right)\prod\limits_{i=1}^d M_i}{\sqrt{N}}\\
	&  + 3M \sqrt{\frac{\log 2/\delta}{2N}}.
	\end{align*}
\end{theorem}\normalsize\vspace{-2mm}

The proof is presented in the supplementary file. As we can see, although we append an extra noise transition adapting element compared with traditional CE loss, the derived generalization error bound is not larger than those derived from the algorithms employing the CE loss, implying that learning with transition matrix does not need extra larger training samples to achieve a good generalization result.

\vspace{-1mm}
\subsection{Algorithm for Estimating $T$}\vspace{-1mm}
Estimation of the optimal $T^*$ and $f^*$ requires two nested loops of optimization (Eq.(\ref{eqfw})(\ref{eqthe})), which is expensive to obtain the exact solution \cite{franceschi2018bilevel}. We thus employ SGD technique, as conventional DNN implementations, to approximately solve our problem in a mini-batch updating manner
\cite{finn2017model,shu2019meta} to jointly ameliorating noise transition $T$ and classifier parameter $\mathbf{W}$ in the DNN classifier $f$.

\textbf{Estimating $T$.} At iteration step $t$,  we firstly adjust the noise transition matrix $T^{(t)}$ according to the classifier parameters $\mathbf{W}^{(t-1)}$ and noise transition matrix $T^{(t-1)}$ obtained in the last step by minimizing the meta loss defined in Eq.(\ref{eqfw}). SGD is employed to optimize the meta loss on a mini-batch containing $m$ meta samples, i.e.,\vspace{-4mm}

\small
\begin{align}\label{theta}
\begin{split}
T^{(t)} = &T^{(t-1)}-\beta \frac{1}{m}\sum_{i=1}^m\\
& \nabla_{T}\mathcal{L}_{M} \left(f(x_i^{(m)};\hat{\mathbf{W}}^{(t)}(T)),y_i^{(m)}\right)\Big|_{T^{(t-1)}},
\end{split}
\end{align}\normalsize

where the following equation is used to formulate $\hat{\mathbf{W}}^{(t)}(T)$ on a mini-batch data containing $n$ training samples,\vspace{-6mm}
\small
\begin{align}
\begin{split}
\hat{\mathbf{W}}^{(t)}(T) = &\mathbf{W}^{(t-1)}-\alpha \frac{1}{n}\sum_{i=1}^n \\
&\nabla_{\mathbf{W}} \ell \left(T^Tf(x_i;\mathbf{W}^{(t-1)}),\widetilde{y}_i\right)\Big|_{\mathbf{W}^{(t-1)}}.
\end{split}
\end{align}\normalsize\vspace{-3mm}

The above learning process is inspired by MAML \cite{finn2017model}, and $\alpha,\beta$ represent the step sizes.

\textbf{Updating $\mathbf{W}$.} When obtained the noise transition matrix $T^{(t)}$, the classifier parameters $\mathbf{W}^{(t)}$ can then be updated by:\vspace{-6mm}

\small
\begin{align}\label{wt}
\begin{split}
\mathbf{W}^{(t)} = &\mathbf{W}^{(t-1)}-\alpha \frac{1}{n}\sum_{i=1}^n \\
&\nabla_{\mathbf{W}} \ell \left({T^{(t)}}^Tf(x_i;\mathbf{W}^{(t-1)}),\widetilde{y}_i\right)\Big|_{\mathbf{W}^{(t-1)}}.
\end{split}
\end{align}\normalsize\vspace{-3mm}

The Meta Transition Adaptation learning algorithm can then be summarized in Algorithm \ref{alg1}. All computations of gradients can be efficiently implemented by automatic differentiation
techniques and easily generalized to any deep learning architectures. The algorithm can be easily implemented using popular deep learning frameworks like PyTorch \cite{paszke2019pytorch}.
It can be seen that both the classifier and the noise transition matrix can be gradually ameliorated during the learning process based on their values calculated in the last step, and the noise transition matrix can thus be updated in a stable manner.

\begin{table*}[t] \vspace{-4mm}
	\caption{Test accuracy (\%) of all competing methods on CIFAR-10 and CIFAR-100 under symmetric and asymmetric noise with different noise levels. The best results are highlighted in \textbf{bold}.}\label{table1} \vspace{1mm}
	\centering
	\setlength{\tabcolsep}{1.5mm}
	\begin{scriptsize}
		\begin{tabular}{c|c|c|c|c|c|c|c|c|c|c}
			\toprule
			\multirow{3}{*}{Datasets}              &  \multirow{3}{*}{Methods}    &\multicolumn{5}{c|}{Symmetric Noise}   & \multicolumn{4}{c}{Asymmetric Noise}    \\
			\cline{3-11}
			&            &  \multicolumn{5}{c|}{Noise Rate $\eta$}         & \multicolumn{4}{c}{Noise Rate $\eta$}   \\ \cline{3-11}
			&            & 0               &             0.2 &           0.4    &             0.6 &            0.8  & 0.2  &0.4  &0.6  &0.8 \\ \hline \hline
			\multirow{9}{*}{CIFAR-10}&     CE     &94.16$\pm$0.25   & 86.38$\pm$0.99 & 77.52$\pm$0.41 & 73.63$\pm$0.85  & 50.31$\pm$2.14  & 83.60$\pm$0.24 & 77.85$\pm$0.98 & 69.69$\pm$0.72 & 55.20$\pm$0.28\\
			&Fine-tuning    &94.40$\pm$0.14   &87.47$\pm$0.80  &   82.23$\pm$0.44 &78.10$\pm$0.59  & 51.44$\pm$3.86& 92.09$\pm$0.14 & 89.96$\pm$0.24 &75.61$\pm$2.91&60.29$\pm$1.46\\
			&  GCE       &  91.73$\pm$0.14 &  89.99$\pm$0.16 &  87.31$\pm$0.53   &  82.15$\pm$0.47  & 57.36$\pm$2.08 &  89.75$\pm$1.53 &87.75$\pm$0.36  & 67.21$\pm$3.64 & 57.46$\pm$0.31 \\
			&     Forward&94.33$\pm$0.31   & 88.26$\pm$0.22 & 83.23 $\pm$0.56 & 78.19$\pm$1.12    & 61.66 $\pm$3.54 & 91.34$\pm$0.28 & 89.87$\pm$0.61 &87.24$\pm$0.96  &81.07$\pm$1.92 \\
			&  GLC       &  94.43$\pm$0.27 & 90.06$\pm$0.30 & 86.78$\pm$0.45   & 82.52$\pm$0.76   &   62.40$\pm$0.14 & 92.87$\pm$0.16 &91.80$\pm$0.24  & 90.95$\pm$0.06 & 90.02$\pm$0.60 \\  
			&  S-Model&94.39$\pm$0.46  & 90.21$\pm$0.14   &  87.92$\pm$2.01    & 81.99$\pm$0.21    &  57.08$\pm$0.23  &    90.86$\pm$0.15        & 84.87$\pm$0.27         & 67.89$\pm$0.46          &  56.17$\pm$1.24         \\
			&T-Revision  &93.86$\pm$0.11   & 90.66$\pm$0.12& 87.88$\pm$0.23   & 83.45$\pm$0.68   & 57.94$\pm$1.56 & 92.48$\pm$0.28 & 91.76$\pm$0.12 & 89.20$\pm$0.69 &84.04$\pm$1.13\\ 
			&MW-Net&93.90$\pm$0.15  & 90.90$\pm$0.66   & 87.02$\pm$0.86  & 82.98$\pm$0.30   & 65.43$\pm$1.51& 92.69$\pm$0.24 &90.17$\pm$0.11  & 68.55$\pm$0.76 & 58.29$\pm$1.33\\
			&Ours    &  \textbf{94.65$\pm$0.03  }        &\textbf{92.54$\pm$0.17}  & \textbf{89.73$\pm$0.41}   & \textbf{ 85.97$\pm$0.10} & \textbf{72.41$\pm$0.32}   & \textbf{93.65$\pm$0.05}   & \textbf{ 93.17$\pm$0.13}  & \textbf{92.57$\pm$0.18}  & \textbf{91.57$\pm$0.28}     \\  \hline \hline
			\multirow{9}{*}{CIFAR-100}& CE        &   76.10$\pm$0.24 &60.38$\pm$0.75   &46.92$\pm$0.51  & 31.82$\pm$1.16  & 8.29$\pm$3.24    &61.05$\pm$0.11  &50.30$\pm$1.11  &37.34$\pm$1.80  & 12.46$\pm$0.43\\
			&Fine-tuning   &76.74$\pm$0.26&64.45$\pm$0.43  &52.69$\pm$1.35 & 38.52$\pm$1.05 &18.95$\pm$0.44 &65.35$\pm$0.80 & 53.11$\pm$0.64 &41.40$\pm$0.43 &19.63$\pm$0.30\\
			&GCE        & 71.97$\pm$0.45 & 68.02$\pm$1.05   &64.18$\pm$0.30  & 54.46$\pm$0.31   & 15.61$\pm$0.97 & 66.15$\pm$0.44 & 56.85$\pm$0.72 & 40.58$\pm$0.47 & 15.82$\pm$0.63  \\
			&Forward    &   76.45$\pm$0.03& 63.71$\pm$0.49   &49.34$\pm$0.60  & 37.90$\pm$0.76  &9.57$\pm$1.01&  64.97$\pm$0.47           &52.37$\pm$0.71  & 44.58$\pm$0.60 & 15.84$\pm$0.62\\
			&GLC        & 76.55$\pm$0.07  & 66.30$\pm$0.62 &59.25$\pm$0.69   &50.86$\pm$0.57 &15.07$\pm$0.78 & 70.83$\pm$0.25 &66.47$\pm$0.58  &54.82$\pm$0.99  & 28.18$\pm$1.88\\
			&S-Model     & 73.69$\pm$0.18& 64.61$\pm$0.95  &60.36$\pm$0.45   & 35.88$\pm$4.47  &7.61$\pm$0.82&66.64$\pm$0.44  & 52.26$\pm$0.17     & 42.96$\pm$0.18   &14.95$\pm$0.60 \\
			& T-Revision &76.12$\pm$0.26  &  68.52$\pm$0.52       &   61.56$\pm$0.37        & 42.48$\pm$0.13          &  7.66$\pm$0.25        &  69.57$\pm$0.12   & 61.80$\pm$0.41     & 44.54$\pm$1.62        &   17.10$\pm$0.22        \\
						&MW-Net&74.93$\pm$0.42 &69.95$\pm$0.40   &65.45$\pm$0.45   & 55.42$\pm$1.36  &21.37$\pm$0.56& 66.73$\pm$0.34 &59.53$\pm$0.40  &52.24$\pm$0.95  & 17.41$\pm$0.52\\
			&Ours & \textbf{76.75$\pm$0.09} &  \textbf{72.58$\pm$0.13}  &  \textbf{68.77$\pm$0.17}     & \textbf{57.85$\pm$0.51}    & \textbf{21.78$\pm$0.42}  & \textbf{74.74$\pm$0.08  } & \textbf{71.58$\pm$0.15}         &    \textbf{61.16$\pm$0.43}     &  \textbf{33.31$\pm$0.78  }      \\
			\bottomrule
		\end{tabular} \vspace{-5mm}
	\end{scriptsize}
\end{table*}

\vspace{-2mm}
\section{Experimental Results} \label{experiment}\vspace{-2mm}
To evaluate the capability of the proposed algorithm, we implement simulated experiments on CIFAR-10, CIFAR-100, TinyImageNet, as well as a large-scale real-world noisy dataset Clothing1M.
\vspace{-2mm}\subsection{Experimental Setup}\vspace{-1mm}
\textbf{Datasets.} We first verify the effectiveness of our method on two benchmark datasets: CIFAR-10 and CIFAR-100 \cite{krizhevsky2009learning}, consisting of $32\times32$ color images arranged in 10 and 100 classes, respectively. Both datasets contain 50,000 training and 10,000 test images. We randomly select 1,000 clean images in the validation set as meta data. We also verify our method on a larger and harder dataset called Tiny-ImageNet (T-ImageNet briefly), containing 200 classes with 100K training, 10K validation, 10K test images of $64\times64$. We randomly sample 10 clean images per class as meta data. These datasets are popularly used for evaluating learning with noisy labels in previous literatures \cite{patrini2017making,goldberger2016training}.

\textbf{Noise setting.} We test two types of label noises: symmetric and asymmetric (class-dependent) noise. \textbf{Symmetric} label noises are generated by flipping the labels of a given proportion of training samples to one of the other class labels uniformly \cite{zhang2016understanding}. Under \textbf{asymmetric} noises, for CIFAR-10, we use the setting in \cite{yao2019safeguarded}. Concretely, we set a probability $r$ to disturb the label to its similar class, i.e., truck $\rightarrow$ automobile, bird $\rightarrow$ airplane, deer $\rightarrow$ horse, cat $\rightarrow$ dog. For CIFAR-100, a similar $r$ is set but the label flip only happens in each super-class as described in \cite{hendrycks2018using}. For T-ImagNet, we adopt the noise setting in \cite{yu2019does}, where labelers also make mistakes only within very similar classes. The graph illustration of asymmetric noise about CIFAR-10 and T-ImageNet can be found in supplementary file.

\begin{table}\vspace{-2mm}
	\caption{Test accuracy (\%) on T-ImageNet under symmetric and asymmetric noise. The best results are in \textbf{bold}. }\label{tableT} \vspace{1mm}
	\centering
	\setlength{\tabcolsep}{1.5mm}
	\begin{scriptsize}
		\begin{tabular}{c|c|c|c|c|c|c|c}
			\toprule
			\multirow{3}{*}{Methods}    &\multicolumn{4}{c|}{Symmetric Noise}   & \multicolumn{3}{c}{Asymmetric Noise} \\
			\cline{2-8}
			&  \multicolumn{4}{c|}{Noise Rate $\eta$}         & \multicolumn{3}{c}{Noise Rate $\eta$}   \\
			\cline{2-8}
			&        0            &        0.2    &       0.4         &     0.6             &     0.2      &       0.4           &   0.6     \\	
			\hline
			\hline
			CE  &     54.10 & 43.94   & 35.14     &  16.45     & 45.83    &  34.95          & 16.24     \\
			Fine-tuning &54.52  & 45.69    &38.06   &  16.60   &  48.57 & 37.17  &18.79\\
			GCE&   50.20      & 46.77    & 41.27  &   19.38   &  47.05   &  34.24    & 14.85    \\
			Forward  &   54.17   & 46.40   & 37.11     &  24.98     & 49.08    &  37.71     &  19.90      \\
			GLC &  54.28      &   48.71  &  42.46     &  25.50     &  49.66  & 40.57   & 31.19    \\
			S-Model &   54.32   & 46.88   & 37.12     &   22.81    &   47.01      &   32.94     &   16.70      \\
			T-Revision&  51.79   & 41.70        &37.04    &  26.44      &   49.63  & 35.02          &   18.87    \\
			MW-Net &   53.58   &48.31   & 43.33      & \textbf{32.23}      &  50.14     &   35.68     &  18.97       \\		
			Ours&      \textbf{54.54 }      &  \textbf{49.85}       & \textbf{43.35}      &     29.22         &       \textbf{51.12}        &   \textbf{43.51 }       &   \textbf{36.32}           \\
			\bottomrule
		\end{tabular}\vspace{-5mm}
	\end{scriptsize}
\end{table}

\textbf{Baselines.} The compared methods include: 1) \textbf{CE}, which uses CE loss to train the DNNs on noisy datasets. 2) \textbf{Fine-tuning}, which finetunes the result of CE on the meta-data to further enhance its performance; 3) \textbf{GCE} \cite{zhang2018generalized}, which employs a robust loss combining the benefits of both CE loss and mean absolute error loss against label noise. 4) \textbf{Forward} \cite{patrini2017making},  which estimates the noise transition matrix in an unsupervised manner.  5) \textbf{GLC} \cite{hendrycks2018using}, which estimates the noise transition matrix by using a small set clean label dataset. 6) \textbf{S-Model }\cite{goldberger2016training}, which uses a Softmax layer to model the noise transition matrix. 7) \textbf{T-Revision} \cite{xia2019anchor}, which learns the noise transition matrix by adding a slack variable to adjust the initialized matrix. 8)\textbf{MW-Net} \cite{shu2019meta}, which uses a MLP net to learn the weighting function in a data-driven fashion. The meta-data in these methods are used as validation set except for Fine-tuning and MW-Net. Note that above 4\&5, 6\&7, 8 methods represent the SOTA one-stage and two-stage noise transition estimation methods, and the SOTA meta-learning method for solving robust DL issue on noisy samples.

\textbf{Network structure.} We use ResNet-34 \cite{he2016deep} as our classifier network for CIFAR-10 and CIFAR-100 dataset followed by \cite{patrini2017making,xia2019anchor}, and a 18-layer Preact ResNet \cite{he2016deep} for T-ImageNet.

\textbf{Experimental setup.} We train the models with SGD, at an initial learning rate $0.0001$ and a momentum 0.9, a weight decay $1\times10^{-3}$ with mini-batch size 128.
The learning rate decays 0.1 at 80 and 100 epochs for a total of 120 epochs. We initialize the softmax parameters of our algorithm with the estimation results of GLC.
\vspace{-3mm}\subsection{Evaluation on Robustness Performance}\vspace{-2mm}
\textbf{Results on CIFAR-10 and CIFAR-100.} The classification accuracies of CIFAR-10 and CIFAR-100 under symmetric and asymmetric noise are reported in Table \ref{table1} with 5 random runs.  As can be seen, our proposed algorithm achieves the best performance in all cases except for CIFAR-100 80\%  symmetric noise. Specifically, even with large noise ratio, our algorithm still shows the competitive classification accuracy. For example, when $\eta=0.8$ on CIFAR-10 symmetric noise and $\eta=0.6$ on CIFAR-100 asymmetric noise, our algorithm reaches 72.41\% and 61.16\%, outperforming the best results of baselines by about 10\% and 6\%, respectively. This demonstrates the robustness of our method on different types and portions of noise.

From Table \ref{table1} it can be found that: 1) Our algorithm evidently improves the performance of Forward and GLC especially in large noise cases, possibly conducted by the inaccurate pre-assumed anchor points, which should be infeasible in real cases. Comparatively, our algorithm can dynamically adjust the transition matrix to make its estimation gradually ameliorated guided by meta data, though our method has a initialization result of GLC. 2) S-model behaves well when noise ratio is small, while degrades quickly when noise ratio becomes large, as well as T-Revision does. This can be explained by the fact that large noise makes it easy to fall into a wrong estimation, as illustrated in Section \ref{under} and Table.\ref{tableM}. Though sharing the same initializations with them, our method can avoid to fall into a wrong estimation and still perform well through being guided by meta data to avoid being trapped by noisy  samples. Especially, when $\eta=0.8$ on CIFAR-100 symmetric noise, both of them underperform the CE methods, while our method achieves a pretty improvement. 3) MW-Net produces a competitive result under the symmetric noise compared with our algorithm. However, it degrades the performance quickly under the asymmetric noise, since for this method, all classes share one weighting function, which is unreasonable when noise is asymmetric. Instead, our method can adaptively fit different noise types and noise rates and gradually ameliorate the estimation. 4) It is interesting to see that our method performs better than CE and fine-tuning even under no-noise scenarios. We will discuss this phenomenon in the next section.

\begin{table}\vspace{-3mm}
	\caption{Comparison of estimation error for noise transition matrix under the asymmetric noise experiments on CIFAR-10 and CIFAR-100 learned by Forward, S-Model, GLC, T-Revision and our method, respectively. S-Model, T-Revision and our method share the same initialized values and the reported results are calculated using matrices learned at last epoch. The estimation error for the matrix is calculated by $\|T-\hat{T}\|_1 /\|T\|_1$, where $T$ and $\hat{T}$ denote the ground-truth and estimated matrices, respectively. }\label{tableM} \vspace{1mm}
	\centering
	\setlength{\tabcolsep}{1.5mm}
	\begin{scriptsize}
		\begin{tabular}{c|c|c|c|c|c|c|c|c}
			\toprule
			\multirow{3}{*}{Methods}    &\multicolumn{4}{c|}{CIFAR-10}   & \multicolumn{4}{c}{CIFAR-100} \\
			\cline{2-9}
			&  \multicolumn{4}{c|}{Noise Rate $\eta$}         & \multicolumn{4}{c}{Noise Rate $\eta$}   \\
			\cline{2-9}
			&        0.2          &        0.4    &       0.6         &     0.8            &     0.2      &       0.4           &   0.6   &0.8  \\	
			\hline
			\hline
			Forward  &   0.163   & 0.197 & 0.209    &  0.342   & 0.446   & 0.701 &  0.727  & 1.691 \\
			GLC&   0.051   &  0.093    & 0.163 &  0.206  &  0.251 &  0.515  &  0.563  & 0.676\\
			S-Model &   0.233   &  0.278   & 0.297   &   0.363  &   1.071  & 1.355    &   1.539   & 1.806\\
			T-Revision &  0.081   & 0.120 & 0.195    &  0.265     &  0.346 & 0.795&1.257 &  1.699 \\
			Ours & \textbf{ 0.046}  &\textbf{0.058}  & \textbf{0.068 }     & \textbf{0.097 }  &  \textbf{ 0.188} &  \textbf{0.273 }   & \textbf{0.297}   & \textbf{0.323}   \\
			\bottomrule
		\end{tabular}\vspace{-5mm}
	\end{scriptsize}
\end{table}

\textbf{Results on  T-ImageNet.} To verify our method on more complex scenario, we summarize in Table \ref{tableT} the test accuracy on T-ImageNet with different noise settings. As we can see, similar to the CIFAR experiments, for both noise settings with different noise rates, our algorithm outperforms all other baselines except for 60\% symmetric noise, where MW-Net beats our algorithm, where all methods have actually lost efficacy. But when the MW-Net is used in more complicated asymmetric noise case with the same noise extent, the method is largely degenerated, where our method can still perform consistently well. The robustness of our method can thus be further substantiated.

\vspace{-2mm}\subsection{How noise transition matrix adapt}\vspace{-2mm}
To understand how our algorithm automatically adjust noise transition matrix guided by the meta data, Table.\ref{tableM} summarizes the estimation error for the transition matrix of the compared methods and ours. It can be observed that our method is more efficient in estimating the transition matrix. Specifically, the matrices learned by Forward and GLC are worse than ours, since the anchor points they find are likely to be inexact, and our method can improve the inexact estimation of GLC towards the groud-truth solution guided by the meta data. On the other hand, although shared the same initialized values with ours, matrices learned by S-Model are easier to fall into a bad estimation when noise ratio increases, leading to poor performance compared with ours. T-revision is also towards bad direction, while the deterioration is slowed down with the control of the revision. Besides, T-Revision deteriorates faster on CIFAR-100 than on CIFAR-10.
Therefore, the estimating matrices by our method are more accurate, naturally following its more robust performance than compared methods.
\begin{table}[t] \vspace{-3mm}
	\caption{Test accuracy (\%) of different models on real-world noisy dataset Clothing1M. The best results are in \textbf{bold}.}\label{Table4} \vspace{1mm}
	\centering
	\setlength{\tabcolsep}{0.8mm}
	\begin{scriptsize}
		\begin{tabular}{c|c|c|c|c|c|c|c|c}
			\toprule
			Methods & CE    &GCE& Forward  & GLC  &  S-Model &    T-Revision   & MW-Net & Ours\\   \hline
			Accuracy & 68.94 & 69.75 &   70.83 & 74.26  &  70.36 &        74.18  &   73.72                &       \textbf{75.59}        \\
			\bottomrule
		\end{tabular} \vspace{-6mm}
	\end{scriptsize}
\end{table}

\vspace{-1mm}\subsection{Experiments on Real-world Noisy Dataset} \label{real-experiment}\vspace{-1mm}
We then verify the applicability of our algorithm on a real-world large-scale noisy dataset: Clothing1M \cite{xiao2015learning}, which contains 1 million images of clothing from online shopping websites with 14 classes, e.g., T-shirt, Shirt, Knitwear. The labels are generated by the surrounding text of images and are thus extremely noisy. The dataset also provides 50k, 14k, 10k manually refined clean data for training, validation and testing, respectively, but we did not use the 50k clean data and use the validation dataset as the meta dataset. Following the previous works \cite{patrini2017making,tanaka2018joint}, we used ResNet-50 pre-trained on ImageNet.  For preprocessing, we resize the image to $256\times256$, crop the center $224\times224$ as input, and perform normalization.  We train the model using SGD with a momentum 0.9, a weight decay $10^{-3}$, an initial learning rate 0.0001, and batch size 100. The learning rate is divided by 10 after 5 epochs (for a total 10 epochs).

The results are summarized in Table \ref{Table4} in terms of top-1 accuracy. Our method outperfoms all baselines. Fig. \ref{fig1} shows some examples of top-5 predictions produced by CE and our method. It can be seen that the top-1 prediction of CE method overfits to the noisy annotations (red labels), while the second top prediction implies the latent clean labels (green labels), reflecting the ambiguity of the sample labels of this dataset. Comparatively, our method can finely recover the true labels through taking the merit of the learned noise transition matrix. For example, the label of the first row image in Fig.\ref{fig1} should be ``T-shirt'', while the annotated label is ``underwear''. The CE method gives 94.2\% confidence to underwear, which is completely trapped by noisy sample. yet our method generates the label ``T-shirt'' with high confidence suppressing the noisy label ``underwear'' benefited from learned noise transition matrix.
\begin{figure}
	\centering
	\includegraphics[width=1.05\linewidth]{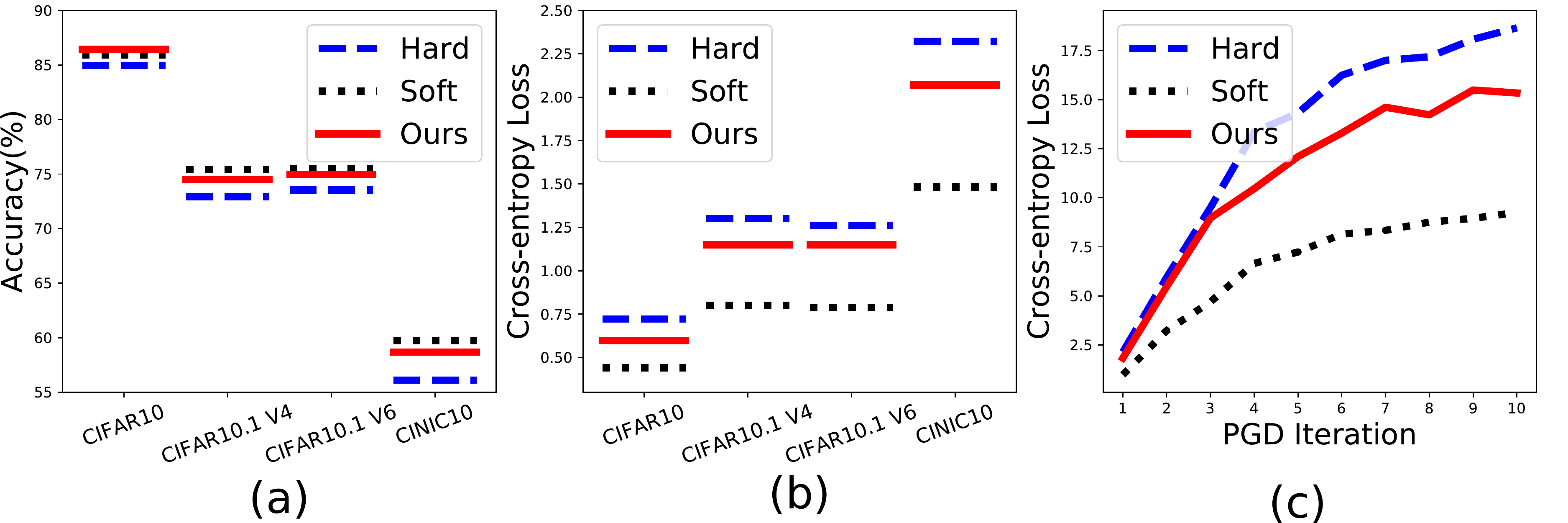}  \vspace{-8mm}
	\caption{Generalization and robustness evaluation results on Hard and Soft labels, as well as ours. (a) accuracy against ground-truth labels, for increasingly out-of-training-sample distributions. (b) CE loss against ground-truth labels. (c) CE loss in PGD iteration.}\vspace{-6mm}
	\label{generalization}
\end{figure}

\vspace{-2mm}

\section{Relation to Label Distribution Learning}\label{ldl}\vspace{-1mm}
It can be observed that our method outperforms CE and Fine-tuning in Table.\ref{table1} and \ref{tableT} even in the no-noise cases, which might be attributed to its intrinsic label distribution learning (LDL) capability \cite{geng2016label,peterson2019human}.
LDL is firstly proposed by \cite{geng2013facial}, which extends the single-label and multi-label annotation to a distribution. \citet{hinton2015distilling} used knowledge distillation to provide the smoothed softmax probabilities to enhance the performance of the student network. To employ soft labels replacing one-hot encoding hard labels, label smoothing \cite{szegedy2016rethinking} and mixup \cite{zhang2017mixup} techniques have also been proposed. Recently, \citet{peterson2019human} presented a full distribution of human labels dataset, CIFAR10H, and utilized it to help improve the accuracy and robustness of a model compared with hard labels.

When there are no noisy labels, our method can be explained to be able to approximate the ground-truth label distribution. Specifically, the hard labels correspond to the most probable label while lose the full label distribution, i.e., including human allocation of probabilities. Therefore, Eq.(\ref{eqthe}) can be interpreted as that the observed data distribution with hard labels is obtained by transforming the underlying data distribution with full label distribution (soft labels) through the transition matrix $T$. The underlying conditional data distribution should behave robust facing unseen data, i.e., to minimize the CE loss over unobserved data (meta data) to bring better generalization and robustness, as validated in \cite{peterson2019human}. Therefore, minimizing Eq.(\ref{eqfw}) can be considered to search $T$ for helping the classifier $f(X)$ recover the underlying conditional data distribution. Therefore, it is rational that our method outperforms CE and Fine-tuning even with less training samples.

Furthermore, to verify that our method can deliver the knowledge of the latent label distribution, we follow the generalization and robustness experiments in \cite{peterson2019human} to compare with Soft and Hard trained with human uncertainty soft labels and one-hot hard labels. The results are demonstrated in Fig. \ref{generalization} and Table \ref{Table5}. For generalization experiment (Section 5 in \cite{peterson2019human}), we train ResNet-110 on 9,900 test images and treat left 100 images randomly chosen 10 images per class as meta data, and evaluate on CIFAR-10 50,000 training set, CIFAR10.1v6,v4 dataset \cite{recht2018cifar} and CINIC10 dataset \cite{darlow2018cinic}. The accuracy of our method is very near to the Soft labels, as seen in Fig. \ref{generalization}(a),  and the CE metric\footnote{The metric is used to evaluate how confident the top prediction of a model is, and whether its distribution over alternative categories is sensible} is evidently better than Hard labels, as seen in Fig. \ref{generalization}(b). These results show our method can improve the generalization of the calculated classifier when test datasets are increasingly out-of-distribution compared with Hard labels.

For robustness experiment, we pretrain ResNet-110 on 49,900 CIFAR-10 training images with treat left 100 images randomly chosen 10 images per class as meta data and then fine-tune pretrained model using 10,000 CIFAR-10 test images. The FGSM attack results \cite{kurakin2016adversarial} are reported in Table \ref{Table5}, averaged over all 10,000 images in CIFAR10 test set. Note that our method obtains higher accuracy and lower CE loss than Hard labels. Fig.\ref{generalization}(c) plots the increase in CE loss for each training scheme conditions on PGD attacks \cite{madry2017towards}. The accuracy was driven to 0\% for Hard labels and ours, and 1\% for Soft labels. However, loss for Hard labels is driven up more rapidly than ours. These results show that our method can also improve the robustness of model compared with Hard labels.
\begin{table}[t]  \vspace{-3mm}
	\caption{Accuracy and cross-entropy loss after FGSM attacks on the networks learned by hard and soft annotations, as well as ours.}\label{Table5} \vspace{1mm}
	\centering
	\setlength{\tabcolsep}{1.5mm}
	\begin{small}
		\begin{tabular}{c|c|c|c|c|c}
			\toprule
			\multicolumn{3}{c|}{Accuracy} & \multicolumn{3}{c}{Cross-entropy}  \\ \hline
			Hard  &  Soft  &  Ours & 		Hard  &  Soft  &  Ours \\  \hline
			26.97& 31.43&\textbf{32.84}  &4.03&\textbf{2.68}  &3.72\\
			\bottomrule
		\end{tabular} \vspace{-4mm}
	\end{small}
\end{table}

\vspace{-1mm}\section{Conclusion}\vspace{-1mm}
We have proposed a novel meta-learning method for adaptively extracting transition matrix to guarantee robust deep learning in the presence of noisy labels. Compared with previous methods that require strong anchor point prior assumption or inaccurate estimation misguided by wrong annotation information, the new method is able to yield a more robust and efficient one guided by a small set of meta data. The statistical consistency guarantee of correctly estimating transition matrix can also be proved. Our empirical results show that the proposed method can behave more robust than the SOTA methods. Besides, we discuss the essential relationship with label distribution learning, and our learning strategy is hopeful to improve the generalization and robustness of the model compared with the standard training on hard labels even under no-noise real scenarios due to the inter-class ambiguity generally existed in real data. In future work, we will try to incorporate priors of the noise structure into transition matrix to further enhance the estimation stability, e.g., assuming sparse transition where corruption only happens in super-classes.


\bibliography{example_paper}

\begin{thebibliography}{64}
\providecommand{\natexlab}[1]{#1}
\providecommand{\url}[1]{\texttt{#1}}
\expandafter\ifx\csname urlstyle\endcsname\relax
  \providecommand{\doi}[1]{doi: #1}\else
  \providecommand{\doi}{doi: \begingroup \urlstyle{rm}\Url}\fi

\bibitem[Arazo et~al.(2019)Arazo, Ortego, Albert, O'Connor, and
  McGuinness]{arazo2019unsupervised}
Arazo, E., Ortego, D., Albert, P., O'Connor, N.~E., and McGuinness, K.
\newblock Unsupervised label noise modeling and loss correction.
\newblock In \emph{ICML}, 2019.

\bibitem[Arora et~al.(2018)Arora, Ge, Neyshabur, and Zhang]{arora2018stronger}
Arora, S., Ge, R., Neyshabur, B., and Zhang, Y.
\newblock Stronger generalization bounds for deep nets via a compression
  approach.
\newblock In \emph{ICML}, 2018.

\bibitem[Bartlett \& Mendelson(2002)Bartlett and
  Mendelson]{bartlett2002rademacher}
Bartlett, P.~L. and Mendelson, S.
\newblock Rademacher and gaussian complexities: Risk bounds and structural
  results.
\newblock \emph{Journal of Machine Learning Research}, 3\penalty0
  (Nov):\penalty0 463--482, 2002.

\bibitem[Bartlett et~al.(2006)Bartlett, Jordan, and
  McAuliffe]{bartlett2006convexity}
Bartlett, P.~L., Jordan, M.~I., and McAuliffe, J.~D.
\newblock Convexity, classification, and risk bounds.
\newblock \emph{Journal of the American Statistical Association}, 101\penalty0
  (473):\penalty0 138--156, 2006.

\bibitem[Bartlett et~al.(2017)Bartlett, Foster, and
  Telgarsky]{bartlett2017spectrally}
Bartlett, P.~L., Foster, D.~J., and Telgarsky, M.~J.
\newblock Spectrally-normalized margin bounds for neural networks.
\newblock In \emph{NeurIPS}, 2017.

\bibitem[Bi et~al.(2014)Bi, Wang, Kwok, and Tu]{bi2014learning}
Bi, W., Wang, L., Kwok, J.~T., and Tu, Z.
\newblock Learning to predict from crowdsourced data.
\newblock In \emph{UAI}, 2014.

\bibitem[Chang et~al.(2017)Chang, Learned-Miller, and
  McCallum]{chang2017active}
Chang, H.-S., Learned-Miller, E., and McCallum, A.
\newblock Active bias: Training more accurate neural networks by emphasizing
  high variance samples.
\newblock In \emph{NeurIPS}, 2017.

\bibitem[Chen et~al.(2019)Chen, Liao, Chen, and Zhang]{chen2019understanding}
Chen, P., Liao, B.~B., Chen, G., and Zhang, S.
\newblock Understanding and utilizing deep neural networks trained with noisy
  labels.
\newblock In \emph{ICML}, 2019.

\bibitem[Darlow et~al.(2018)Darlow, Crowley, Antoniou, and
  Storkey]{darlow2018cinic}
Darlow, L.~N., Crowley, E.~J., Antoniou, A., and Storkey, A.~J.
\newblock Cinic-10 is not imagenet or cifar-10.
\newblock \emph{arXiv preprint arXiv:1810.03505}, 2018.

\bibitem[Finn et~al.(2017)Finn, Abbeel, and Levine]{finn2017model}
Finn, C., Abbeel, P., and Levine, S.
\newblock Model-agnostic meta-learning for fast adaptation of deep networks.
\newblock In \emph{ICML}, 2017.

\bibitem[Franceschi et~al.(2018)Franceschi, Frasconi, Salzo, Grazzi, and
  Pontil]{franceschi2018bilevel}
Franceschi, L., Frasconi, P., Salzo, S., Grazzi, R., and Pontil, M.
\newblock Bilevel programming for hyperparameter optimization and
  meta-learning.
\newblock In \emph{ICML}, 2018.

\bibitem[Geng(2016)]{geng2016label}
Geng, X.
\newblock Label distribution learning.
\newblock \emph{IEEE Transactions on Knowledge and Data Engineering},
  28\penalty0 (7):\penalty0 1734--1748, 2016.

\bibitem[Geng et~al.(2013)Geng, Yin, and Zhou]{geng2013facial}
Geng, X., Yin, C., and Zhou, Z.-H.
\newblock Facial age estimation by learning from label distributions.
\newblock \emph{IEEE transactions on pattern analysis and machine
  intelligence}, 35\penalty0 (10):\penalty0 2401--2412, 2013.

\bibitem[Glorot et~al.(2011)Glorot, Bordes, and Bengio]{glorot2011deep}
Glorot, X., Bordes, A., and Bengio, Y.
\newblock Deep sparse rectifier neural networks.
\newblock In \emph{AISTATS}, 2011.

\bibitem[Goldberger \& Ben-Reuven(2017)Goldberger and
  Ben-Reuven]{goldberger2016training}
Goldberger, J. and Ben-Reuven, E.
\newblock Training deep neural-networks using a noise adaptation layer.
\newblock In \emph{ICLR}, 2017.

\bibitem[Golowich et~al.(2018)Golowich, Rakhlin, and Shamir]{golowich2018size}
Golowich, N., Rakhlin, A., and Shamir, O.
\newblock Size-independent sample complexity of neural networks.
\newblock In \emph{COLT}, 2018.

\bibitem[Han et~al.(2018)Han, Yao, Yu, Niu, Xu, Hu, Tsang, and
  Sugiyama]{han2018co}
Han, B., Yao, Q., Yu, X., Niu, G., Xu, M., Hu, W., Tsang, I., and Sugiyama, M.
\newblock Co-teaching: Robust training of deep neural networks with extremely
  noisy labels.
\newblock In \emph{NeurIPS}, 2018.

\bibitem[He et~al.(2016)He, Zhang, Ren, and Sun]{he2016deep}
He, K., Zhang, X., Ren, S., and Sun, J.
\newblock Deep residual learning for image recognition.
\newblock In \emph{CVPR}, 2016.

\bibitem[Hendrycks et~al.(2018)Hendrycks, Mazeika, Wilson, and
  Gimpel]{hendrycks2018using}
Hendrycks, D., Mazeika, M., Wilson, D., and Gimpel, K.
\newblock Using trusted data to train deep networks on labels corrupted by
  severe noise.
\newblock In \emph{NeurIPS}, 2018.

\bibitem[Hinton et~al.(2015)Hinton, Vinyals, and Dean]{hinton2015distilling}
Hinton, G., Vinyals, O., and Dean, J.
\newblock Distilling the knowledge in a neural network.
\newblock \emph{arXiv preprint arXiv:1503.02531}, 2015.

\bibitem[Jiang et~al.(2014{\natexlab{a}})Jiang, Meng, Mitamura, and
  Hauptmann]{jiang2014easy}
Jiang, L., Meng, D., Mitamura, T., and Hauptmann, A.~G.
\newblock Easy samples first: Self-paced reranking for zero-example multimedia
  search.
\newblock In \emph{ACM MM}, 2014{\natexlab{a}}.

\bibitem[Jiang et~al.(2014{\natexlab{b}})Jiang, Meng, Yu, Lan, Shan, and
  Hauptmann]{jiang2014self}
Jiang, L., Meng, D., Yu, S.-I., Lan, Z., Shan, S., and Hauptmann, A.
\newblock Self-paced learning with diversity.
\newblock In \emph{NeurIPS}, 2014{\natexlab{b}}.

\bibitem[Jiang et~al.(2018)Jiang, Zhou, Leung, Li, and
  Fei-Fei]{jiang2018mentornet}
Jiang, L., Zhou, Z., Leung, T., Li, L.-J., and Fei-Fei, L.
\newblock Mentornet: Learning data-driven curriculum for very deep neural
  networks on corrupted labels.
\newblock In \emph{ICML}, 2018.

\bibitem[Jindal et~al.(2016)Jindal, Nokleby, and Chen]{jindal2016learning}
Jindal, I., Nokleby, M., and Chen, X.
\newblock Learning deep networks from noisy labels with dropout regularization.
\newblock In \emph{ICDM}, 2016.

\bibitem[Krizhevsky(2009)]{krizhevsky2009learning}
Krizhevsky, A.
\newblock Learning multiple layers of features from tiny images.
\newblock Technical report, 2009.

\bibitem[Krizhevsky et~al.(2012)Krizhevsky, Sutskever, and
  Hinton]{krizhevsky2012imagenet}
Krizhevsky, A., Sutskever, I., and Hinton, G.~E.
\newblock Imagenet classification with deep convolutional neural networks.
\newblock In \emph{NeurIPS}, 2012.

\bibitem[Kumar et~al.(2010)Kumar, Packer, and Koller]{kumar2010self}
Kumar, M.~P., Packer, B., and Koller, D.
\newblock Self-paced learning for latent variable models.
\newblock In \emph{NeurIPS}, 2010.

\bibitem[Kurakin et~al.(2016)Kurakin, Goodfellow, and
  Bengio]{kurakin2016adversarial}
Kurakin, A., Goodfellow, I., and Bengio, S.
\newblock Adversarial examples in the physical world.
\newblock In \emph{ICLR}, 2016.

\bibitem[Ledoux \& Talagrand(1991)Ledoux and Talagrand]{ledoux1991probability}
Ledoux, M. and Talagrand, M.
\newblock \emph{Probability in Banach Spaces: Isoperimetry and Processes},
  volume~23.
\newblock Springer Science \& Business Media, 1991.

\bibitem[Li et~al.(2017)Li, Yang, Song, Cao, Luo, and Li]{li2017learning}
Li, Y., Yang, J., Song, Y., Cao, L., Luo, J., and Li, L.-J.
\newblock Learning from noisy labels with distillation.
\newblock In \emph{ICCV}, 2017.

\bibitem[Liang et~al.(2016)Liang, Jiang, Meng, and
  Hauptmann]{liang2016learning}
Liang, J., Jiang, L., Meng, D., and Hauptmann, A.~G.
\newblock Learning to detect concepts from webly-labeled video data.
\newblock In \emph{IJCAI}, 2016.

\bibitem[Madry et~al.(2018)Madry, Makelov, Schmidt, Tsipras, and
  Vladu]{madry2017towards}
Madry, A., Makelov, A., Schmidt, L., Tsipras, D., and Vladu, A.
\newblock Towards deep learning models resistant to adversarial attacks.
\newblock In \emph{ICLR}, 2018.

\bibitem[Malach \& Shalev-Shwartz(2017)Malach and
  Shalev-Shwartz]{malach2017decoupling}
Malach, E. and Shalev-Shwartz, S.
\newblock Decoupling" when to update" from" how to update".
\newblock In \emph{NeurIPS}, 2017.

\bibitem[Meng et~al.(2017)Meng, Zhao, and Jiang]{meng2017theoretical}
Meng, D., Zhao, Q., and Jiang, L.
\newblock A theoretical understanding of self-paced learning.
\newblock \emph{Information Sciences}, 414:\penalty0 319--328, 2017.

\bibitem[Mohri et~al.(2018)Mohri, Rostamizadeh, and
  Talwalkar]{mohri2018foundations}
Mohri, M., Rostamizadeh, A., and Talwalkar, A.
\newblock \emph{Foundations of Machine Learning}.
\newblock MIT Press, 2018.

\bibitem[Natarajan et~al.(2013)Natarajan, Dhillon, Ravikumar, and
  Tewari]{natarajan2013learning}
Natarajan, N., Dhillon, I.~S., Ravikumar, P.~K., and Tewari, A.
\newblock Learning with noisy labels.
\newblock In \emph{NeurIPS}, 2013.

\bibitem[Neyshabur et~al.(2018)Neyshabur, Bhojanapalli, and
  Srebro]{neyshabur2017pac}
Neyshabur, B., Bhojanapalli, S., and Srebro, N.
\newblock A pac-bayesian approach to spectrally-normalized margin bounds for
  neural networks.
\newblock In \emph{ICLR}, 2018.

\bibitem[Paszke et~al.(2019)Paszke, Gross, Massa, Lerer, Bradbury, Chanan,
  Killeen, Lin, Gimelshein, Antiga, et~al.]{paszke2019pytorch}
Paszke, A., Gross, S., Massa, F., Lerer, A., Bradbury, J., Chanan, G., Killeen,
  T., Lin, Z., Gimelshein, N., Antiga, L., et~al.
\newblock Pytorch: An imperative style, high-performance deep learning library.
\newblock In \emph{NeurIPS}, 2019.

\bibitem[Patrini et~al.(2017)Patrini, Rozza, Krishna~Menon, Nock, and
  Qu]{patrini2017making}
Patrini, G., Rozza, A., Krishna~Menon, A., Nock, R., and Qu, L.
\newblock Making deep neural networks robust to label noise: A loss correction
  approach.
\newblock In \emph{CVPR}, 2017.

\bibitem[Peterson et~al.(2019)Peterson, Battleday, Griffiths, and
  Russakovsky]{peterson2019human}
Peterson, J.~C., Battleday, R.~M., Griffiths, T.~L., and Russakovsky, O.
\newblock Human uncertainty makes classification more robust.
\newblock In \emph{ICCV}, 2019.

\bibitem[Recht et~al.(2018)Recht, Roelofs, Schmidt, and
  Shankar]{recht2018cifar}
Recht, B., Roelofs, R., Schmidt, L., and Shankar, V.
\newblock Do cifar-10 classifiers generalize to cifar-10?
\newblock \emph{arXiv preprint arXiv:1806.00451}, 2018.

\bibitem[Ren et~al.(2018)Ren, Zeng, Yang, and Urtasun]{ren2018learning}
Ren, M., Zeng, W., Yang, B., and Urtasun, R.
\newblock Learning to reweight examples for robust deep learning.
\newblock In \emph{ICML}, 2018.

\bibitem[Schmidhuber(1992)]{schmidhuber1992learning}
Schmidhuber, J.
\newblock Learning to control fast-weight memories: An alternative to dynamic
  recurrent networks.
\newblock \emph{Neural Computation}, 4\penalty0 (1):\penalty0 131--139, 1992.

\bibitem[Scott(2015)]{scott2015rate}
Scott, C.
\newblock A rate of convergence for mixture proportion estimation, with
  application to learning from noisy labels.
\newblock In \emph{AISTATS}, 2015.

\bibitem[Scott et~al.(2013)Scott, Blanchard, and
  Handy]{scott2013classification}
Scott, C., Blanchard, G., and Handy, G.
\newblock Classification with asymmetric label noise: Consistency and maximal
  denoising.
\newblock In \emph{Conference On Learning Theory}, pp.\  489--511, 2013.

\bibitem[Shen \& Sanghavi(2019)Shen and Sanghavi]{shen2019learning}
Shen, Y. and Sanghavi, S.
\newblock Learning with bad training data via iterative trimmed loss
  minimization.
\newblock In \emph{ICML}, 2019.

\bibitem[Shu et~al.(2018)Shu, Xu, and Meng]{shu2018small}
Shu, J., Xu, Z., and Meng, D.
\newblock Small sample learning in big data era.
\newblock \emph{arXiv preprint arXiv:1808.04572}, 2018.

\bibitem[Shu et~al.(2019)Shu, Xie, Yi, Zhao, Zhou, Xu, and Meng]{shu2019meta}
Shu, J., Xie, Q., Yi, L., Zhao, Q., Zhou, S., Xu, Z., and Meng, D.
\newblock Meta-weight-net: Learning an explicit mapping for sample weighting.
\newblock In \emph{NeurIPS}, 2019.

\bibitem[Song et~al.(2019)Song, Kim, and Lee]{song2019selfie}
Song, H., Kim, M., and Lee, J.-G.
\newblock Selfie: Refurbishing unclean samples for robust deep learning.
\newblock In \emph{ICML}, 2019.

\bibitem[Sukhbaatar et~al.(2015)Sukhbaatar, Bruna, Paluri, Bourdev, and
  Fergus]{sukhbaatar2014training}
Sukhbaatar, S., Bruna, J., Paluri, M., Bourdev, L., and Fergus, R.
\newblock Training convolutional networks with noisy labels.
\newblock In \emph{ICLR workshop}, 2015.

\bibitem[Szegedy et~al.(2016)Szegedy, Vanhoucke, Ioffe, Shlens, and
  Wojna]{szegedy2016rethinking}
Szegedy, C., Vanhoucke, V., Ioffe, S., Shlens, J., and Wojna, Z.
\newblock Rethinking the inception architecture for computer vision.
\newblock In \emph{CVPR}, 2016.

\bibitem[Tanaka et~al.(2018)Tanaka, Ikami, Yamasaki, and
  Aizawa]{tanaka2018joint}
Tanaka, D., Ikami, D., Yamasaki, T., and Aizawa, K.
\newblock Joint optimization framework for learning with noisy labels.
\newblock In \emph{CVPR}, 2018.

\bibitem[Thrun \& Pratt(1998)Thrun and Pratt]{thrun2012learning}
Thrun, S. and Pratt, L.
\newblock \emph{Learning to learn}.
\newblock Springer, 1998.

\bibitem[Vahdat(2017)]{vahdat2017toward}
Vahdat, A.
\newblock Toward robustness against label noise in training deep discriminative
  neural networks.
\newblock In \emph{NeurIPS}, 2017.

\bibitem[Van~Rooyen \& Williamson(2017)Van~Rooyen and
  Williamson]{van2017theory}
Van~Rooyen, B. and Williamson, R.~C.
\newblock A theory of learning with corrupted labels.
\newblock \emph{Journal of Machine Learning Research}, 18:\penalty0 228--1,
  2017.

\bibitem[Wang et~al.(2017)Wang, Kucukelbir, and Blei]{wang2017robust}
Wang, Y., Kucukelbir, A., and Blei, D.~M.
\newblock Robust probabilistic modeling with bayesian data reweighting.
\newblock In \emph{ICML}, pp.\  3646--3655, 2017.

\bibitem[Xia et~al.(2019)Xia, Liu, Wang, Han, Gong, Niu, and
  Sugiyama]{xia2019anchor}
Xia, X., Liu, T., Wang, N., Han, B., Gong, C., Niu, G., and Sugiyama, M.
\newblock Are anchor points really indispensable in label-noise learning?
\newblock In \emph{NeurIPS}, 2019.

\bibitem[Xiao et~al.(2015)Xiao, Xia, Yang, Huang, and Wang]{xiao2015learning}
Xiao, T., Xia, T., Yang, Y., Huang, C., and Wang, X.
\newblock Learning from massive noisy labeled data for image classification.
\newblock In \emph{CVPR}, 2015.

\bibitem[Yao et~al.(2019)Yao, Wu, Zhang, Tsang, and Sun]{yao2019safeguarded}
Yao, J., Wu, H., Zhang, Y., Tsang, I.~W., and Sun, J.
\newblock Safeguarded dynamic label regression for noisy supervision.
\newblock In \emph{AAAI}, 2019.

\bibitem[Yin et~al.(2019)Yin, Kannan, and Bartlett]{yin2019rademacher}
Yin, D., Kannan, R., and Bartlett, P.
\newblock Rademacher complexity for adversarially robust generalization.
\newblock In \emph{ICML}, 2019.

\bibitem[Yu et~al.(2019)Yu, Han, Yao, Niu, Tsang, and Sugiyama]{yu2019does}
Yu, X., Han, B., Yao, J., Niu, G., Tsang, I., and Sugiyama, M.
\newblock How does disagreement help generalization against label corruption?
\newblock In \emph{ICML}, 2019.

\bibitem[Zhang et~al.(2017{\natexlab{a}})Zhang, Bengio, Hardt, Recht, and
  Vinyals]{zhang2016understanding}
Zhang, C., Bengio, S., Hardt, M., Recht, B., and Vinyals, O.
\newblock Understanding deep learning requires rethinking generalization.
\newblock In \emph{ICLR}, 2017{\natexlab{a}}.

\bibitem[Zhang et~al.(2017{\natexlab{b}})Zhang, Cisse, Dauphin, and
  Lopez-Paz]{zhang2017mixup}
Zhang, H., Cisse, M., Dauphin, Y.~N., and Lopez-Paz, D.
\newblock mixup: Beyond empirical risk minimization.
\newblock In \emph{ICLR}, 2017{\natexlab{b}}.

\bibitem[Zhang \& Sabuncu(2018)Zhang and Sabuncu]{zhang2018generalized}
Zhang, Z. and Sabuncu, M.
\newblock Generalized cross entropy loss for training deep neural networks with
  noisy labels.
\newblock In \emph{NeurIPS}, 2018.

\end{thebibliography}
\bibliographystyle{icml2020}

\newpage
\appendix

\section{Solution of Estimating Noise Transition }
In our paper, we jointly learn the noise transition matrix and classifier by miniming the following bi-level optimization problems \cite{franceschi2018bilevel,shu2019meta}
\begin{align}\label{eq1}
T^* = \mathop{\arg\min}_{T\in [0,1]^{c\times c}} \frac{1}{M} \sum_{i=1}^M \mathcal{L}_{M} (f^*_T(x_i^{(m)}),y_i^{(m)}),\\ \label{eq2}
f^*_T=
\mathop{\arg\min}_{f\in \mathcal{F}}\mathbb{E}_{(X,\widetilde{Y})\sim P_{X\widetilde{Y}}} \ell (T(\Theta)^Tf(X),\widetilde{Y}).
\end{align}
The empirical version of above can be written as follows used in our main paper:
\begin{align} \label{eq3}
T^* = \mathop{\arg\min}_{T\in [0,1]^{c\times c}} \frac{1}{M} \sum_{i=1}^M \mathcal{L}_{M} (f^*_T(x_i^{(m)}),y_i^{(m)}),\\ \label{eq4}
f^*_T=
\mathop{\arg\min}_{f\in \mathcal{F}}-\frac{1}{N}\sum_{i=1}^N \ell (T^Tf(x_i),\widetilde{y}_i),
\end{align}

We try to illustrate that the theoretical solution of above optimization problems recover the solution we require.

\begin{lemma}\label{lemma1}
	Suppose $\ell$ is the cross-entropy loss, and $f(X) \in \Delta^{c-1}$, i.e., $\sum_{i=1}^c f_i(X)=1$. Then by minimizing the expected risk $R(f)=\mathbb{E}_{(X,Y)\sim P_{XY}} \ell (f(X),Y)$, the optimal mapping $f^*$ satisfies $f^*_i(X)=p(Y=i|X), \forall i\in[c]$.
\end{lemma}
\begin{proof}
	Minimizing the expected risk $R(f)$ can be written as 
	\begin{align}
	\begin{split}
	\min_{f \in \mathcal{F}}\ \mathcal{\phi}(f) &= -\sum_{i=1}^c p(Y=i|X) \log f_i(X),\\
	&s.t. \sum_{i=1}^{c} f_i(X)=1.
	\end{split}
	\end{align}
	By using Lagrange Multiplier method, we have 
	\begin{align}
	f^* = \mathop{\arg\min}_{f \in \mathcal{F}}  \mathcal{L}(f):= \mathcal{\phi}(f) -\lambda\left(\sum_{i=1}^{c} f_i(X)-1\right).
	\end{align}
	Take the erivative of $\mathcal{L}(f)$ witth respect to $f$, we have $\frac{\partial \mathcal{L}(f)}{\partial f_i^*}=0$. Thus, we have 
	\begin{align}
	f_i^*(X)=-\lambda p(Y=i|X), \forall i \in [c], \forall X \in \mathcal{X}.
	\end{align}
	Since $\sum_{i=1}^{c} f_i(X)=1$ and $\sum_{i=1}^{c} p(Y=i|X)=1$, we can easily obtain $\lambda=-1$. Therefore, we have
	\begin{align}
	f^*_i(X)=p(Y=i|X), \forall i\in[c],\forall X \in \mathcal{X}.
	\end{align}
\end{proof}

\begin{theorem}
	Suppose both of the training and meta loss used in our paper are cross-entropy loss and the meta data are i.i.d. drawn from clean data.
	Then the solution $T^*$ by minimizing Eq.(\ref{eq1})(\ref{eq2}) can recover the ground-truth noise transition matrix in a certain probability .
\end{theorem}
\begin{proof}
	The expected risk on clean data is defined as \cite{bartlett2006convexity}:
	\begin{align}
	R(f) = \mathbb{E}_{(X,Y)\sim P_{XY}} \ell (f(X),Y),
	\end{align}
	and the empirical risk over meta dataset $D_{Meta}$ is defined as:
	\begin{align}
	R_M(f) = \frac{1}{M}\sum_{i=1}^M \ell (f(x_i^{(m)}),y_i^{(m)}).
	\end{align}
	Since meta dataset can be seen as i.i.d. sampling from clean data, we can deduce that by Hoeffding’s inequality,
	$\forall \ 0<\varepsilon<1, $,  the following holds for all $f\in \mathcal{F}$ with probability at least $1-\delta$
	\begin{align}
	R_M(f)-\sqrt{\frac{\ln(2/\delta)}{2M}}\leq R(f) \leq R_M(f)+\sqrt{\frac{\ln(2/\delta)}{2M}}.
	\end{align}
	We denote $T^*$ and $T^G$ as the learned transition matrix by minimizing Eq.(\ref{eq1})(\ref{eq2}) and the underlying transition matrix, respectively. We calculate $R(f_{T^*})-R(f_{T^G})$ to character the difference between $T^*$ and $T^G$, since $T^G$ is unavailable. Since
	\begin{align*}
	&R(f_{T^*})-R(f_{T^G}) = R(f_{T^*})-R_M(f_{T^*})\\
	&+R_M(f_{T^*})-R_M(f_{T^G})+R_M(f_{T^G})-R(f_{T^G}),
	\end{align*}
	we have the following holds for all $f\in \mathcal{F}$ with probability at least $1-\delta$
	\begin{align}\label{eqss}
	\begin{split}
	& R_M(f_{T^*})-R_M(f_{T^G}) -2\sqrt{\frac{\ln(2/\delta)}{2M}}\\
	&\leq R(f_{T^*})-R(f_{T^G}) \leq\\
	& R_M(f_{T^*})-R_M(f_{T^G}) +2\sqrt{\frac{\ln(2/\delta)}{2M}}
	\end{split}
	\end{align}
	thus $R_M(f_{T^*})-R_M(f_{T^G})$ can control the approximation degree of $R(f_{T^*})-R(f_{T^G})$. Subsequently, we will show that minimizing Eq.(\ref{eq1})(\ref{eq2}) can make $R_M(f_{T^*})-R_M(f_{T^G})$ as small as possible.

	We provide the proof by contradiction. Suppose that the optimal solution  $T^*$ of Eq.(\ref{eq1}) can not recover the ground-truth noise transition matrix,  we can show that $f_{T^*}$ obtained by optimizing Eq.(\ref{eq2}) still overfits to the label noise. Otherwise, when $f_{T^*}$ recovers the clean classifier, we have ${T^*}^T f_{T^*}\neq p(\widetilde{Y}|X)$. However, by Lemma \ref{lemma1}, the minimization of Eq.(\ref{eq1})(\ref{eq2}) pushes that ${T^*}^T f_{T^*}= p(\widetilde{Y}|X)$ holds.
	This means that $f_{T^*}$ can not recover the classifier on the clean data $p(Y|X)$. 
	Thus $R_M(f_{T^*})$ can not get the best performance. Then minimizing Eq.(\ref{eq1}) pushes $R_M(f_{T^*})$ as small as possible until $T^*$ approaches to $T^G$, i.e., pushes $R_M(f_{T^*})-R_M(f_{T^G})$ as small as possible.
	
	Based on Eq.(\ref{eqss}), $R(f_{T^*})-R(f_{T^G})$ can be bounded by minimizing Eq.(\ref{eq1})(\ref{eq2}). In other words, Eq.(\ref{eqss}) holds for all $f\in \mathcal{F}$ with probability at least $1-\delta$. Since $R_M(f_{T^*})-R_M(f_{T^G})$ can be very small, we can deduce that $T^*$ recover $T^G$ in a certain probability.
	The proof is completed.
\end{proof}

\begin{remark}
	Since Eq.({\ref{eq4}}) is the empirical version of Eq.(\ref{eq2}), when $N\rightarrow \infty$, the solution of Eq.({\ref{eq4}}) can approach to the solution of Eq.({\ref{eq2}}). In view of this, the analysis of minimizing Eq.(\ref{eq3})(\ref{eq4}) can be easily incorporated in Theorem 1.
\end{remark}

\section{Generalization Error}
The results in this paper focus on Rademacher complexity \cite{bartlett2002rademacher,mohri2018foundations,yin2019rademacher}, which is a standard tool to control the uniform convergence (and hence the sample complexity) of given classes of predictors. Here, we present its formal definition. For any function class $\mathcal{H}\subseteq \mathbb{R}^{\mathcal{Z}}$, given a sample $\mathcal{D}=\{z_1,z_2,\cdots,z_N\}$ of size $N$, the empirical Rademacher complexity is defined as
\begin{align}
\hat{\mathfrak{R}}_{\mathcal{D}}(\mathcal{H})=\frac{1}{N}\mathbb{E}_{\sigma}\left[\sup_{h\in \mathcal{H}}\sum_{i=1}^N \sigma_i h(z_i)\right],
\end{align}
wher $\sigma_1,\sigma_2,\cdots,\sigma_N$ are i.i.d. Rademacher random variables with $p\{\sigma_i=1\}=p\{\sigma_i=-1\}=\frac{1}{2}$. In our learning problem, denote the training sample by clean dataset $\mathcal{D}= \{(x_i,y_i)\}_{i=1}^N$ and noisy dataset $\widetilde{\mathcal{D}}\{(x_i,\widetilde{y}_i)\}_{i=1}^N$. The expected and empirical risks are $R(f)= \mathbb{E}_{(X,Y)\sim P_{XY}} \ell (f(X),Y)$ and $R_N(f) = \frac{1}{N}\sum_{i=1}^N \ell (f(x_i),y_i)$; We then have the following theorem which connects the expected and empirical risks via Rademacher complexity.
\begin{theorem}[Rademacher Complexity]
	Suppose that the range of the loss function $\ell$ is $[0,M]$. Then for any $\delta \in (0,1)$, with the probability at least $1-\delta$, the following holds for all $f\in\mathcal{F}$:
	\begin{align}
	R(f)\leq R_N(f)+ 2M \hat{\mathfrak{R}}_{\mathcal{D}}(\ell \circ \mathcal{F}) + 3M \sqrt{\frac{\log 2/\delta}{2N}},
	\end{align}
	where $\hat{\mathfrak{R}}_{\mathcal{D}}(\ell \circ \mathcal{F})=\mathbb{E}_{\sigma}[\sup_{f\in\mathcal{F}} \frac{1}{N}\sum_{i=1}^{N} \sigma_i \ell(f(x_i),y_i) ]$ is the Rademacher complexity; $\{\sigma_1,\sigma_2,\cdots,\sigma_N\}$ are Rademacher variables uniformly distributed from $\{-1,1\}$.
\end{theorem}

In our paper, our goal is to minimize the following expected risk and the empirical risk with respected to noisy data to recover the unbias classifier,
\begin{align}\label{eq33}
\overline{R}(f) &= \mathbb{E}_{(X,\widetilde{Y})\sim P_{X\widetilde{Y}}} \ell (T^Tf(X),\widetilde{Y}), \\
\overline{R}_N(f)&=-\frac{1}{N}\sum_{i=1}^N \ell (T^Tf(x_i),\widetilde{y}_i)
\end{align}
Therefore, the Rademacher complexity for our problems can be expressed as follows:
\begin{corollary}
	For any $\delta \in (0,1)$, with the probability at least $1-\delta$, the following holds for all $f\in\mathcal{F}$:
	\begin{align*}
	\overline{R}(f)\leq \overline{R}_N(f)+ 2M \hat{\mathfrak{R}}_{\widetilde{\mathcal{D}}}(\hat{\ell} \circ \mathcal{F}) + 3M \sqrt{\frac{\log 2/\delta}{2N}},
	\end{align*}
	where $\hat{\ell} \circ \mathcal{F} = \{\ell (T^Tf(X),\widetilde{Y}): f\in\mathcal{F}\}$.
\end{corollary}
Here, the argument $f\in \mathcal{F}$ in the Rademacher complexity $\hat{\mathfrak{R}}_{\widetilde{\mathcal{D}}}(\hat{\ell} \circ \mathcal{F})$ indicates that $f$ is chosen from the function
space $\mathcal{F}$, which is generally determined by the function space of $h$ due to the fact that $f_i(x)=\frac{\exp(\bm{h}_i(x))}{\sum_{k=1}^c\exp(\bm{h}_k(x))},i=1,2,\cdots,c$. Thus, we have the following conclusion.
\begin{proposition}
	$\hat{\mathfrak{R}}_{\widetilde{\mathcal{D}}}(\hat{\ell} \circ \mathcal{F})	\leq c\ \hat{\mathfrak{R}}_{\widetilde{\mathcal{D}}}( \mathcal{H})$, where $\mathcal{H}$ denotes the hypothesis complexity of the classifier.
\end{proposition}

\begin{proof}
	Firstly, we provide the following two lemmas related to our proof.
	\begin{lemma}
		The loss function $\ell(T^Tf(X),\widetilde{Y}=i)$ is 1-Lipschitz with respect to $h_j(X), \forall j\in [c]$, where $\ell$ is cross-entropy loss. 
	\end{lemma}
	\begin{proof}
		Since $f_j(x)=\frac{\exp(\bm{h}_j(x))}{\sum_{l=1}^c\exp(\bm{h}_l(x))},j=1,2,\cdots,c$, we have
		\begin{align*}
		\ell(T^Tf(X),\widetilde{Y}=i) = -\log \left(\sum_{k=1}^c T_{kj}\frac{\exp(h_j)}{\sum_{l=1}^c\exp{h_l}}\right). 
		\end{align*}
		Take the derivative of $\ell(T^Tf(X),\widetilde{Y}=i)$ with respect to $h_j(X)$, we have
		\begin{align*}
		&\frac{\partial \ell(T^Tf(X),\widetilde{Y}=i)}{\partial h_j(X)} \\
		&=-\frac{T_{ji}\exp(h_j(X))}{\sum\limits_{k=1}^c T_{ki} \exp(h_k(X))}+\frac{\exp(h_j(X))}{\sum\limits_{l=1}^c \exp(h_l(X))}
		\end{align*}
		Thus, we have
		\begin{align*}
		&\frac{\partial \ell(T^Tf(X),\widetilde{Y}=i)}{\partial h_j(X)} \leq  \frac{(1-T_{ji})\exp(h_j(X))}{\sum\limits_{l=1}^c  \exp(h_k(X))} \\
		&\leq   \frac{\exp(h_j(X))}{\sum\limits_{l=1}^c  \exp(h_k(X))}\leq 1,
		\end{align*}
		and 
		\begin{align*}
		\frac{\partial \ell(T^Tf(X),\widetilde{Y}=i)}{\partial h_j(X)} \geq -\frac{T_{ji}\exp(h_j(X))}{\sum\limits_{k=1}^c T_{ki} \exp(h_k(X))} \geq -1
		\end{align*}
		Therefore, we can demonstrate that the loss function is 1-Lipschitz with respect to $h_j(X), \forall j\in [c]$.
	\end{proof}
	Since $\hat{\mathfrak{R}}_{\widetilde{\mathcal{D}}}(\hat{\ell} \circ \mathcal{F})$ depends on the loss function, while not the hypothesis space. Talagrand's Contraction Lemma  \cite{ledoux1991probability,mohri2018foundations} tries connect both of them.
	
	\begin{lemma}[Talagrand's Contraction Lemma ] 
		Let $\Phi:\mathbb{R}\rightarrow \mathbb{R}$ be an $L$-Lipschitz function. Then, for any hypothesis set $\mathcal{H}$ of real-valued functions, we have
		\begin{align}
		\hat{\mathfrak{R}}_{S}(\Phi \circ \mathcal{H})\leq L  \hat{\mathfrak{R}}_{S}(\mathcal{H}).
		\end{align}
	\end{lemma}

	Now, we can proof the conclusion.
	\begin{small}
		\begin{align*}
		&\hat{\mathfrak{R}}_{\widetilde{\mathcal{D}}}(\hat{\ell} \circ \mathcal{F}) = \frac{1}{N}\mathbb{E}_{\sigma}\left[\sup_{f\in \mathcal{F}}\sum_{i=1}^N \sigma_i \ell (T^Tf(X_i),\widetilde{Y}_i)\right]\\
		&= \frac{1}{N}\mathbb{E}_{\sigma}\left[\sup_{h_j\in \mathcal{H}}\sum_{i=1}^N \sigma_i \sum_{j=1}^c \ell (\sum_{k=1}^cT_{kj} \frac{\exp(h_k(X_i))}{\sum\limits_{l=1}^c \exp(h_l(X_i))},\widetilde{Y}_i=j)\right]\\
		&= \frac{1}{N}\sum_{j=1}^c\mathbb{E}_{\sigma}\left[\sup_{h_j\in \mathcal{H}}\sum_{i=1}^N \sigma_i  \ell (\sum_{k=1}^cT_{kj} \frac{\exp(h_k(X_i))}{\sum\limits_{l=1}^c \exp(h_l(X_i))},\widetilde{Y}_i=j)\right]\\
		& \leq  \sum_{j=1}^c \frac{1}{N}\mathbb{E}_{\sigma}\left[\sup_{h_j\in \mathcal{H}} \sum_{i=1}^N \sigma_i h_j(X_i) \right]\\
		& = c\ \hat{\mathfrak{R}}_{\widetilde{\mathcal{D}}}(\mathcal{H}),
		\end{align*}
	\end{small}
	where the inequality holds since the Talagrand’s Contraction Lemma.
\end{proof}

Notice that $\hat{\mathfrak{R}}_{\widetilde{\mathcal{D}}}(\mathcal{H})$ measures the hypothesis complexity of deep neural networks, which has been widely studied recently \cite{bartlett2017spectrally,neyshabur2017pac,arora2018stronger,golowich2018size}. Here, we directly use the following Theorem in \cite{golowich2018size} to measure $\hat{\mathfrak{R}}_{\widetilde{\mathcal{D}}}(\mathcal{H})$.
\begin{theorem}\label{TH2}
	Let $\mathcal{H}$ be the class of real-valued networks of depth $d$ over the domain $\mathcal{X}=\{x: \|x\|\leq B\}$, where each parameter matrix $W^{(i)}$ has Frobenius norm at most $M_i$, and the activation function $\sigma$ is 1-Lipschitz, positive-homogeneous and applied element-wise (such as the ReLU). Then 
	\begin{align}
	\hat{\mathfrak{R}}_{\widetilde{\mathcal{D}}}(\mathcal{H}) \leq \frac{B\left(\sqrt{2\log(2)d}+1\right)\prod\limits_{i=1}^d M_i}{\sqrt{N}}.
	\end{align}
\end{theorem}	

Combined with the above results, we have the following Theorem to prove tight bounds for the Rademacher complexity of our problems.
\begin{theorem}\label{TH3}
	Let $\mathcal{H}$ be the class of real-valued networks of depth $d$ over the domain $\mathcal{X}=\{x: \|x\|\leq B\}$, where each parameter matrix $W^{(i)}$ has Frobenius norm at most $M_i$, and the activation function $\sigma$ is 1-Lipschitz, positive-homogeneous and applied element-wise (such as the ReLU).  Let the loss function be the cross-entropy loss and $T,f$ be the learned noise transition matrix and classifier according to Eq.(\ref{eq3})(\ref{eq4}). Then for any $\delta \in (0,1)$, with the probability at least $1-\delta$, the following holds:
	\begin{align*}
	\overline{R}(f)\leq& \overline{R}_N(f)+  \frac{2cMB\left(\sqrt{2\log(2)d}+1\right)\prod\limits_{i=1}^d M_i}{\sqrt{N}}\\
	&  + 3M \sqrt{\frac{\log 2/\delta}{2N}}.
	\end{align*}
\end{theorem}

Furthermore, we can estimate the error between $f$ and the optimal classifier $f^*$ of Eq.(\ref{eq33}), which is also the optimal classifier on the clean data.
The error is estimated via upper bounding $\overline{R}(f) -\overline{R}(f^*)$.

\begin{corollary}
	Under the same conditions as Theorem \ref{TH3}, let $f^*$ be the optimal classifier learned on the clean data. Then for any $\delta \in (0,1)$, with the probability at least $1-\delta$, the following holds:
	\begin{align*}
	\begin{split}
	&\overline{R}(f) -\overline{R}(\overline{f}^*)\leq \\
	&\frac{4cMB\left(\sqrt{2\log(2)d}+1\right)\prod\limits_{i=1}^d M_i}{\sqrt{N}}  + 6M \sqrt{\frac{\log 2/\delta}{2N}},
	\end{split}
	\end{align*}
\end{corollary}
\begin{proof}
	\begin{align*}
	&\overline{R}(f) -\overline{R}(f^*)\\ =\ &\overline{R}(f)-\overline{R}_N(f)+\overline{R}_N(f)-\overline{R}_N(f^*)+\overline{R}_N(f^*) -\overline{R}(f^*)\\
	\leq\ & \overline{R}(f)-\overline{R}_N(f) +\overline{R}_N(f^*) -\overline{R}(f^*)\\
	\leq \ &2\sup_{f\in \mathcal{F}} |\overline{R}(f)- \overline{R}_N(f)|\\
	\leq \ & \frac{4cMB\left(\sqrt{2\log(2)d}+1\right)\prod\limits_{i=1}^d M_i}{\sqrt{N}}  + 6M \sqrt{\frac{\log 2/\delta}{2N}},
	\end{align*} 
	where the first inequality holds since $\overline{R}_N(f)-\overline{R}_N(f^*)\leq 0$.
\end{proof}

\begin{table*}[t] 
	\caption{Test accuracy (\%) of our method using GLC as initialization and GLC on CIFAR-10 and CIFAR-100 under symmetric and asymmetric noise with different noise levels. The best results are highlighted in \textbf{bold}.}\label{table11} \vspace{1mm}
	\centering
	\setlength{\tabcolsep}{1mm}
	\begin{tiny}
		\begin{tabular}{c|c|c|c|c|c|c|c|c|c|c}
			\toprule
			\multirow{2}{*}{Datasets}              &  \multirow{3}{*}{Methods}    &\multicolumn{5}{c|}{Symmetric Noise}   & \multicolumn{4}{c}{Asymmetric Noise}    \\
			\cline{3-11}
			&            &  \multicolumn{5}{c|}{Noise Rate $\eta$}         & \multicolumn{4}{c}{Noise Rate $\eta$}   \\ \cline{3-11}
			&            & 0               &             0.2 &           0.4    &             0.6 &            0.8  & 0.2  &0.4  &0.6  &0.8 \\ \hline \hline
			\multirow{2}{*}{CIFAR-10}&      GLC       &  94.43$\pm$0.27 & 90.06$\pm$0.30 & 86.78$\pm$0.45   & 82.52$\pm$0.76   &   62.40$\pm$0.14 & 92.87$\pm$0.16 &91.80$\pm$0.24  & 90.95$\pm$0.06 & 90.02$\pm$0.60 \\  
			&Ours\_G    &  \textbf{94.65$\pm$0.03  }        &\textbf{92.54$\pm$0.17}  & \textbf{89.73$\pm$0.41}   & \textbf{ 85.97$\pm$0.10} & \textbf{72.41$\pm$0.32}   & \textbf{93.65$\pm$0.05}   & \textbf{ 93.17$\pm$0.13}  & \textbf{92.57$\pm$0.18}  & \textbf{91.57$\pm$0.28}     \\  \hline \hline
			\multirow{2}{*}{CIFAR-100}& GLC        & 76.55$\pm$0.07  & 66.30$\pm$0.62 &59.25$\pm$0.69   &50.86$\pm$0.57 &15.07$\pm$0.78 & 70.83$\pm$0.25 &66.47$\pm$0.58  &54.82$\pm$0.99  & 28.18$\pm$1.88\\
			&Ours\_G & \textbf{76.75$\pm$0.09} &  \textbf{72.58$\pm$0.13}  &  \textbf{68.77$\pm$0.17}     & \textbf{57.85$\pm$0.51}    & \textbf{21.78$\pm$0.42}  & \textbf{74.74$\pm$0.08  } & \textbf{71.58$\pm$0.15}         &    \textbf{61.16$\pm$0.43}     &  \textbf{33.31$\pm$0.78  }      \\
			\bottomrule
		\end{tabular} \vspace{-4mm}
	\end{tiny}
\end{table*}

\begin{table*}[t] 
	\caption{Test accuracy (\%) of our method using Forward as initialization and Forward on CIFAR-10 and CIFAR-100 under symmetric and asymmetric noise with different noise levels. The best results are highlighted in \textbf{bold}.}\label{table22} \vspace{1mm}
	\centering
	\setlength{\tabcolsep}{1mm}
	\begin{tiny}
		\begin{tabular}{c|c|c|c|c|c|c|c|c|c|c}
			\toprule
			\multirow{2}{*}{Datasets}              &  \multirow{3}{*}{Methods}    &\multicolumn{5}{c|}{Symmetric Noise}   & \multicolumn{4}{c}{Asymmetric Noise}    \\
			\cline{3-11}
			&            &  \multicolumn{5}{c|}{Noise Rate $\eta$}         & \multicolumn{4}{c}{Noise Rate $\eta$}   \\ \cline{3-11}
			&            & 0               &             0.2 &           0.4    &             0.6 &            0.8  & 0.2  &0.4  &0.6  &0.8 \\ \hline \hline
			\multirow{2}{*}{CIFAR-10}&        Forward&94.33$\pm$0.31   & 88.26$\pm$0.22 & 83.23 $\pm$0.56 & 78.19$\pm$1.12    & 61.66 $\pm$3.54 & 91.34$\pm$0.28 & 89.87$\pm$0.61 &87.24$\pm$0.96  &81.07$\pm$1.92  \\  
			& Ours\_F    &\textbf{94.52$\pm$0.03   }      &\textbf{ 92.19$\pm$0.07}& \textbf{89.61$\pm$0.18} &  \textbf{85.60$\pm$0.30} &  \textbf{68.49$\pm$0.33}&  \textbf{93.51$\pm$0.04}& \textbf{92.98$\pm$0.20} &\textbf{92.30$\pm$0.04 }&\textbf{ 91.20$\pm$0.31 }\\  \hline \hline
			\multirow{2}{*}{CIFAR-100}& Forward    &   76.45$\pm$0.03& 63.71$\pm$0.49   &49.34$\pm$0.60  & 37.90$\pm$0.76  &9.57$\pm$1.01&  64.97$\pm$0.47           &52.37$\pm$0.71  & 44.58$\pm$0.60 & 15.84$\pm$0.62\\
			& Ours\_F          &   \textbf{76.71$\pm$0.05}&  \textbf{70.74$\pm$0.28} &   \textbf{65.88$\pm$1.06}  & \textbf{57.10$\pm$0.58}& \textbf{17.68$\pm$0.11}  & \textbf{71.99$\pm$0.17 }   & \textbf{67.60$\pm$0.32}& \textbf{57.88$\pm$0.85 }& \textbf{16.76$\pm$0.30 } \\
			\bottomrule
		\end{tabular} \vspace{-4mm}
	\end{tiny}
\end{table*}

\section{Discussion on initialized strategy of our method}
The noise transition of our method, as well as comparison methods S-Model \cite{goldberger2016training} and T-Revision \cite{xia2019anchor}, is initialized by estimation results of GLC. It can be seen from Table \ref{table11} that our method (Ours\_G) obtains significant improvement than GLC's original results in all noise cases.
To evaluate the stability of our method on initialization, we have also used unsupervised Forward \cite{patrini2017making} as initialization in our method (F-Ini). Compared with Forword's original performance as shown in Table \ref{table22}, our method (Ours\_F) can also bring evident improvement. Comparatively, G-Ini attains better final performance than F-Ini in most cases. We thus suggest using GLC's result as initialization in our paper.

\section{Noise Transition Matrix in Our Paper}
The graph illustration of asymmetric noise about CIFAR-10 and T-ImageNet we use in our paper can be found in Fig. \ref{fig2}.

\begin{figure}[htb]
	\centering
	\subfigcapskip=-1.5mm
	\subfigure[CIFAR-10]{
		\label{fig2a} 
		\includegraphics[width=0.22\textwidth]{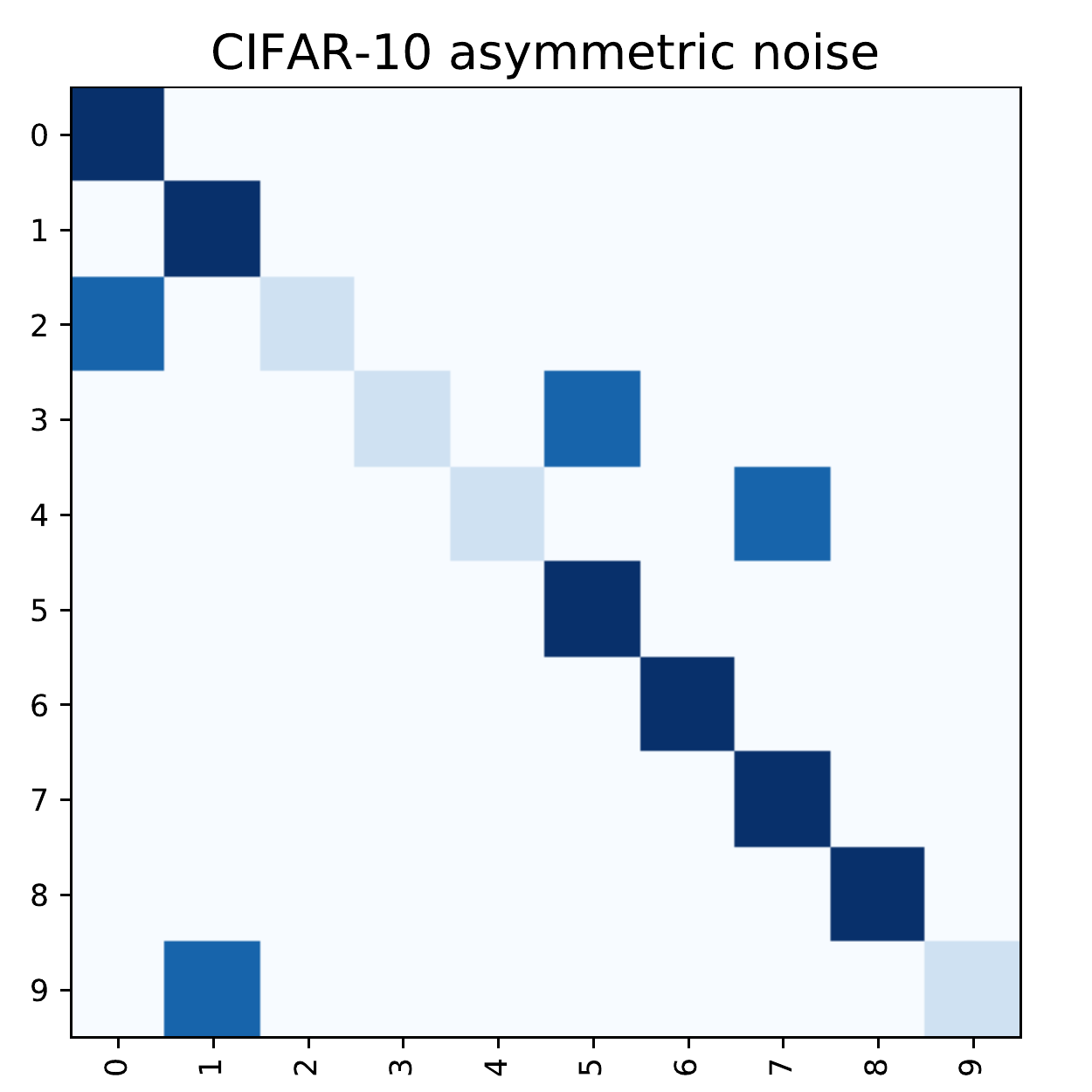}} \ 
	\subfigure[T-ImageNet]{
		\label{fig2b} 
		\includegraphics[width=0.22\textwidth]{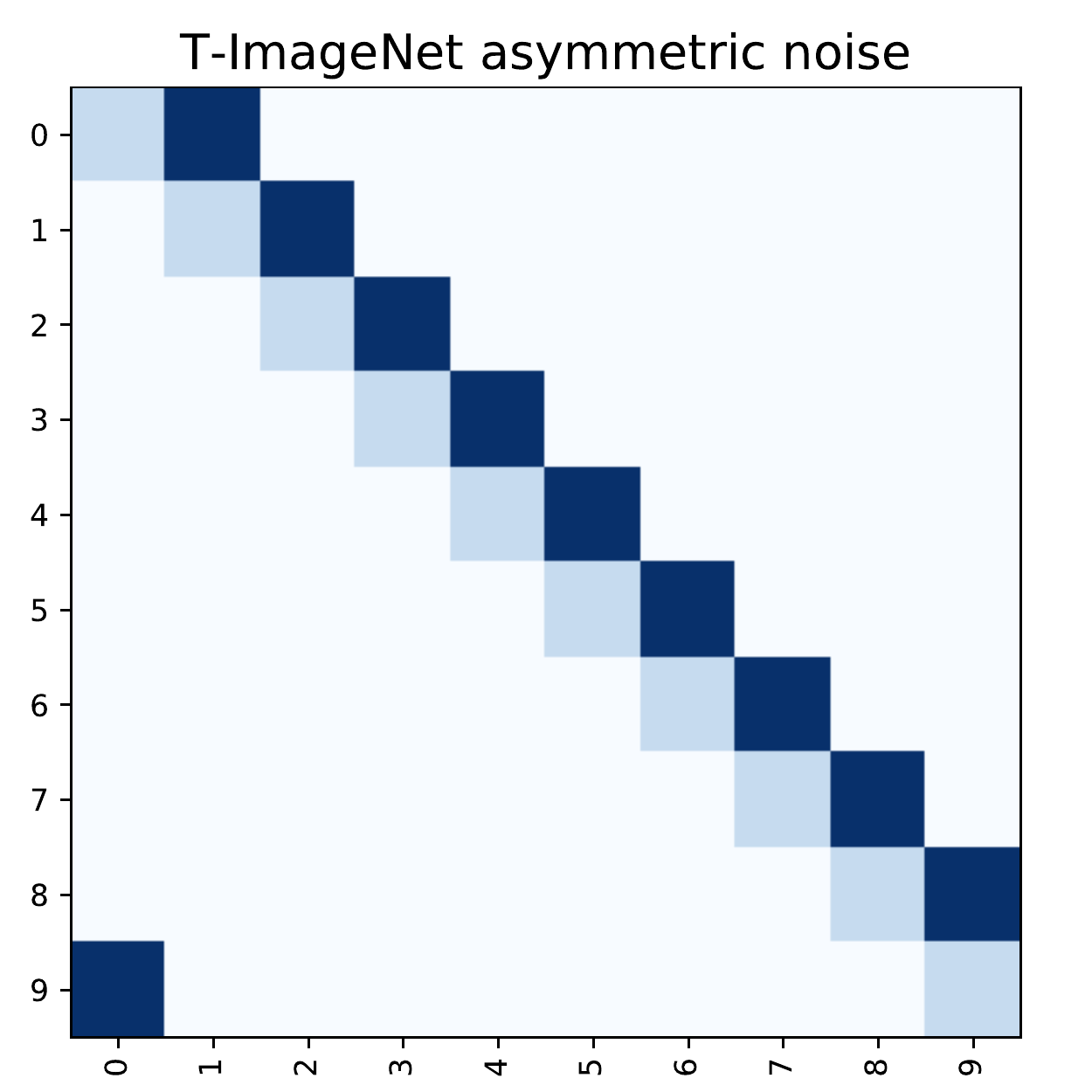}} \vspace{-4mm}
	\caption{Graph illustration of asymmetric noise about CIFAR-10 and T-ImageNet we use under 80\% noise ratio. }\label{fig2}
	\vspace{-5mm}
\end{figure}

\end{document}